\documentclass[sigconf,pbalance,nonacm]{acmart}

\usepackage{booktabs}
\usepackage{colortbl}
\usepackage{multirow}
\usepackage{amsmath}
\usepackage{array}
\usepackage{algorithm}
\usepackage{algorithmic}
\usepackage{enumitem}
\usepackage{xspace}
\usepackage{subcaption}
\usepackage{afterpage}
\usepackage[capitalise]{cleveref}

\keywords{diffusion transformers, training-free acceleration, feature prediction, linear multistep methods, image generation, video generation}

\begin{document} 

\title[PrediT]{Predict to Skip: Linear Multistep Feature Forecasting for Efficient Diffusion Transformers}

\author{\fontsize{9}{10}\selectfont Hanshuai Cui$^{2,1}$, Zhiqing Tang$^{1}$, Qianli Ma$^{1}$, Zhi Yao$^{2,1}$, Weijia Jia$^{1}$, Wei Zhao$^{3}$}
\affiliation{%
  \institution{$^{1}$Institute of Artificial Intelligence and Future Networks, Beijing Normal University\city{Zhuhai}\country{China}}
}
\affiliation{%
  \institution{$^{2}$School of Artificial Intelligence, Beijing Normal University\city{Beijing}\country{China}}
}
\affiliation{%
  \institution{$^{3}$Shenzhen University of Advanced Technology\city{Shenzhen}\country{China}}
}
\renewcommand{\shortauthors}{Cui et al.}

\begin{abstract}

Diffusion Transformers (DiT) have emerged as a widely adopted backbone for high-fidelity image and video generation, yet their iterative denoising process incurs high computational costs. Existing training-free acceleration methods rely on feature caching and reuse under the assumption of temporal stability. However, reusing features for multiple steps may lead to latent drift and visual degradation. As recognized in recent work, model outputs evolve smoothly along much of the diffusion trajectory. We leverage this property to enable principled predictions rather than naive reuse. Based on this insight, we propose PrediT, a training-free acceleration framework that formulates feature prediction as a linear multistep problem. We employ classical linear multistep methods to forecast future model outputs from historical information, combined with a corrector that activates in high-dynamics regions to prevent error accumulation. A dynamic step modulation mechanism adaptively adjusts the prediction horizon by monitoring the feature change rate. Together, these components enable substantial acceleration while preserving generation fidelity. Extensive experiments validate that our method achieves up to 5.54x latency reduction across various DiT-based image and video generation models, while incurring negligible quality degradation. Our code is publicly available at \url{https://github.com/hsc113/PrediT}.
\end{abstract}

\maketitle

\section{Introduction}

\begin{figure}[t]
 \centering
 \includegraphics[width=\linewidth]{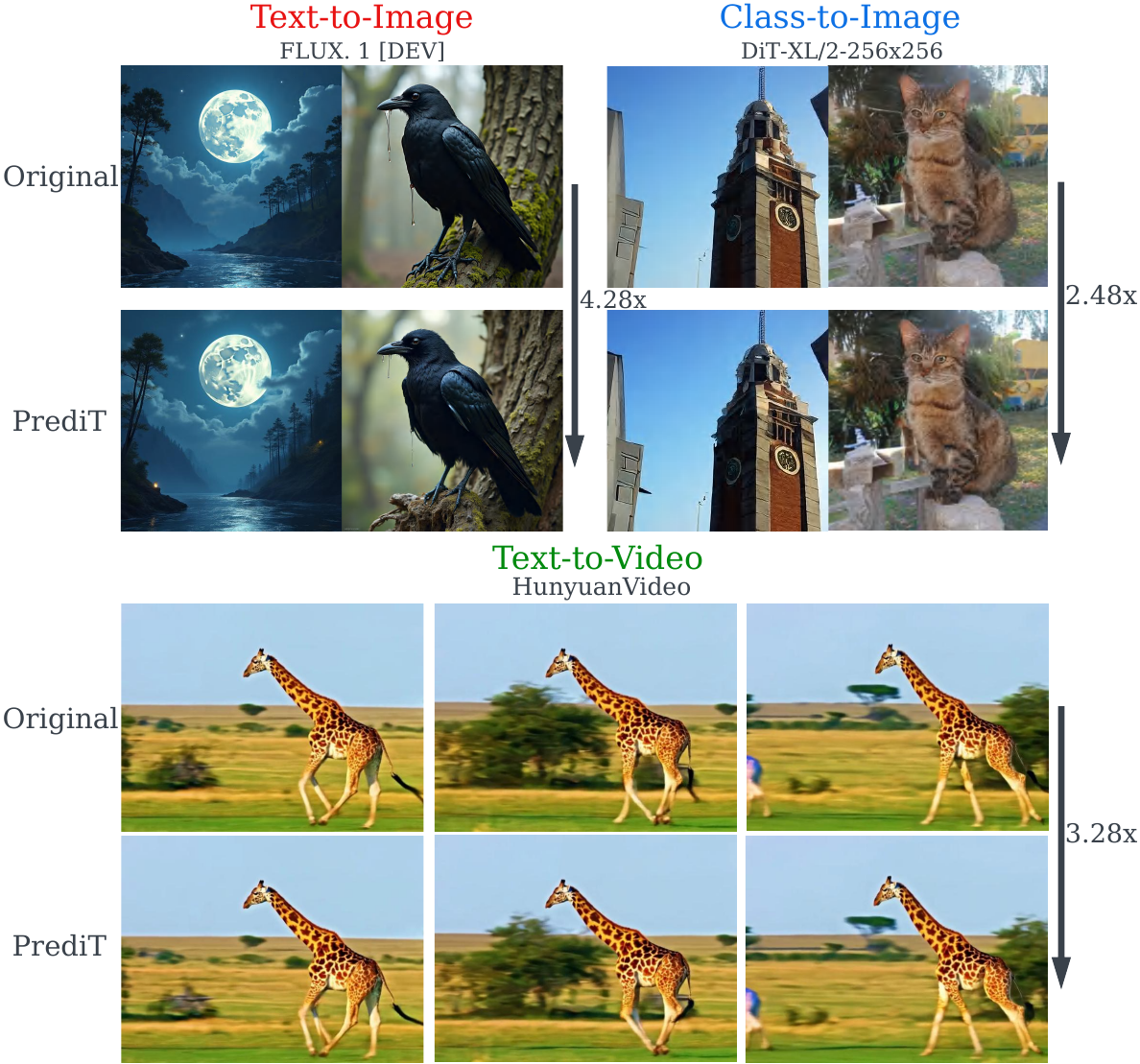}
 \caption{Accelerating Diffusion Transformer inference across multiple models. PrediT achieves significant speedup on various DiT-based architectures for both image and video generation while preserving visual quality.}
 \Description{A teaser montage of image and video generation examples from several Diffusion Transformer models, illustrating that PrediT reduces inference cost while preserving visual quality.}
 \label{fig:teaser}
\end{figure}

Generative AI has achieved remarkable success across a wide range of applications, such as image and video generation, and controllable content creation~\cite{fui2023generative,cao2025survey}. Diffusion models (DMs)~\cite{dhariwal2021diffusion,ho2020denoising,rombach2022high} generate data by learning to reverse a gradual noising process, becoming the foundation of these advances. Early DMs relied on U-Net~\cite{blattmann2023stable,ramesh2022hierarchical} architectures, but the limited receptive field of convolutional networks makes it difficult to capture long-range dependencies, restricting generation quality for complex scenes. Diffusion Transformers (DiT)~\cite{peebles2023scalable} have been proposed to address this limitation, leveraging the self-attention mechanism to model global relationships across the entire latent space. DiT models are now a widely adopted backbone behind state-of-the-art models~\cite{opensora,chen2023pixart,chen2024pixart,liu2024sora,esser2024scaling}.
However, DiT inference is computationally expensive, as the quadratic attention cost combined with iterative denoising results in significant latency, making real-time generation impractical.

To accelerate DiT inference, two main categories of methods have been explored. Training-based approaches, such as distillation~\cite{salimans2022progressive,meng2023distillation} and quantization~\cite{shang2023post,li2023q}, reduce model complexity but require substantial training data and compute, and may degrade generation quality due to distribution shift. Training-free methods are increasingly popular due to their plug-and-play nature. Among these, caching-based approaches exploit the temporal redundancy in diffusion sampling by caching intermediate features and reusing them across consecutive steps, thereby reducing the number of full model samplings~\cite{liu2025survey}.

Despite their efficiency, existing caching methods face limitations. Reuse-based methods such as DeepCache~\cite{ma2024deepcache}, FORA~\cite{selvaraju2024fora}, and $\Delta$-DiT~\cite{chen2024delta} directly substitute previous model outputs for current ones, assuming that features remain unchanged across steps. However, this assumption breaks down in regions of high-dynamics along the denoising trajectory, leading to latent drift and potentially causing visual artifacts. To maintain acceptable quality, these methods must limit their acceleration ratio, resulting in modest speedups. More advanced approaches such as AB-Cache~\cite{yu2025ab} and TaylorSeer~\cite{liu2025reusing} extrapolate future features from historical information. While these methods improve upon naive reuse, they still suffer from error accumulation under fixed skip intervals, especially when the model output exhibits non-uniform dynamics.

Building on the observation from recent work~\cite{yu2025ab,liu2025reusing} that the model's feature trajectory along the diffusion process is locally smooth (see \cref{fig:pca}), we argue that this temporal structure can be more effectively exploited beyond naive reuse. To enable high-quality generation with aggressive acceleration, we identify two key challenges that must be addressed: \textbf{(1) How to predict accurately and stably?} Naive reuse is a zero-order approximation, while existing forecasting methods rely on finite-difference estimation that is sensitive to noise and unstable in higher-order terms, causing error accumulation under large skip intervals. \textbf{(2) How to adapt the prediction horizon dynamically?} The feature dynamics vary significantly throughout the diffusion trajectory, with the initial and final phases exhibiting rapid changes while middle timesteps are smooth. A fixed skip schedule cannot account for this variation, leading to over-skipping (quality loss) or under-skipping (limited speedup).

In this paper, we propose \textbf{PrediT} (\textbf{Pre}dictive \textbf{DiT}), a training-free framework that treats feature estimation as a linear multistep prediction problem. We employ the Adams-Bashforth predictor, which directly combines historical function values without explicit derivative estimation, to extrapolate future model outputs with improved numerical stability. An Adams-Moulton corrector is further applied in high-dynamics regions to refine predictions and prevent error accumulation. A dynamic step modulation mechanism monitors the relative change rate and adaptively adjusts how many steps can be safely skipped. In contrast to prior methods that either passively reuse stale features or apply fixed-interval extrapolation without adaptive control, PrediT allows principled prediction with adaptive step modulation, achieving aggressive acceleration without quality degradation. As shown in \cref{fig:teaser}, PrediT achieves $4.28\times$ speedup on FLUX~\cite{flux2024}, $2.48\times$ on DiT-XL/2~\cite{peebles2023scalable}, and $3.28\times$ on HunyuanVideo~\cite{kong2024hunyuanvideo} while maintaining comparable visual fidelity. In summary, our main contributions are:
\begin{itemize}[itemsep=2pt, topsep=2pt, parsep=0pt]
 \item We analyze why naive feature reuse leads to latent drift and show that diffusion trajectories are locally smooth, motivating higher-order polynomial prediction.
 \item We propose PrediT, a training-free framework employing the Adams-Bashforth predictor with Adams-Moulton correction for stable feature forecasting. A dynamic step modulation mechanism monitors the relative feature change rate and adaptively adjusts the prediction interval.
 \item Extensive experiments on several DiT-based models demonstrate that PrediT achieves up to $5.54\times$ speedup while maintaining comparable generation quality, outperforming existing caching methods.
\end{itemize}

\begin{figure*}[t]
 \centering
 \begin{subfigure}[b]{0.29\textwidth}
   \centering
   \includegraphics[width=\linewidth]{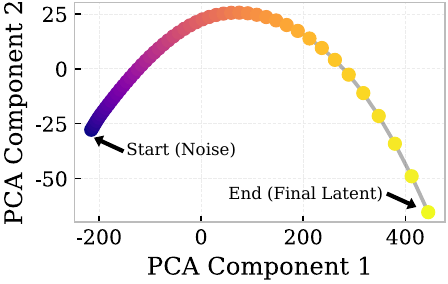}
   \caption{Feature Trajectory}
   \label{fig:pca}
 \end{subfigure}
 \hfill
 \begin{subfigure}[b]{0.705\textwidth}
   \centering
   \includegraphics[width=\linewidth]{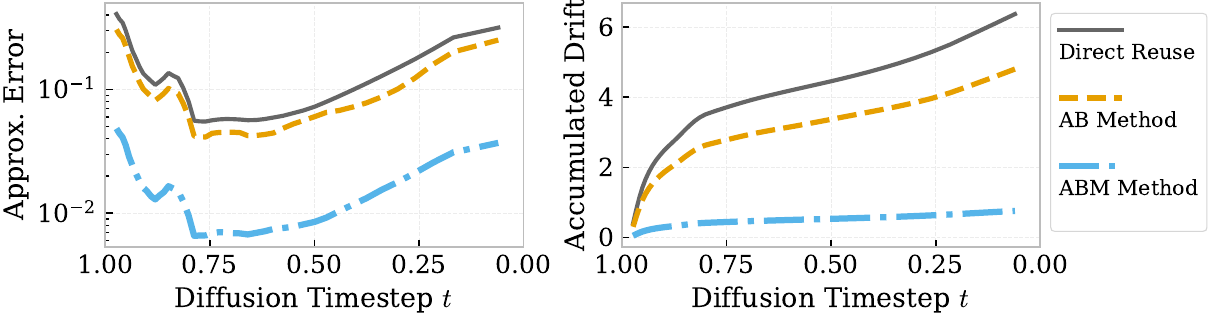}
   \caption{Comparison of Approximation Error and Accumulated Drift}
   \label{fig:metrics_plot}
 \end{subfigure}
 \caption{Analysis of direct feature reuse limitations. (a) Feature trajectory visualization via PCA shows local smoothness along the diffusion process. (b) Comparison of per-step approximation error and accumulated drift: naive reuse incurs significant error accumulation, while higher-order prediction methods reduce drift.}
 \Description{A two-panel analysis figure. The left panel shows a PCA visualization of feature trajectories along the diffusion process. The right panel plots approximation error and accumulated drift, showing that higher-order prediction reduces drift compared with naive reuse.}
 \label{fig:combined_analysis}
\end{figure*}

\section{Related Work}
\noindent\textbf{Diffusion Models.} Diffusion models (DMs)~\cite{dhariwal2021diffusion,ho2020denoising,rombach2022high} have succeeded across diverse generative tasks.
These models generate high-quality samples through progressive denoising and iterative refinement~\cite{croitoru2023diffusion,yang2023diffusion}. Early diffusion models predominantly relied on U-Net architectures. However, they exhibited challenges in capturing long-range dependencies across spatial regions~\cite{chen2023videocrafter1,chen2024videocrafter2}. 
To address these limitations, Diffusion Transformers (DiTs)~\cite{peebles2023scalable} replace the U-Net backbone with a transformer-based architecture~\cite{blattmann2023stable,ramesh2022hierarchical}. Despite their impressive generation quality and scalability advantages, DiTs incur higher inference latency than their predecessors.
This computational burden has motivated extensive research into accelerating DiT inference.

\textbf{Efficient Diffusion Models.}
Numerous approaches have been proposed to accelerate diffusion models.
DDIM~\cite{song2020denoising} introduces a non-Markovian formulation that allows deterministic sampling with significantly fewer steps. 
Model complexity reduction methods primarily focus on distillation~\cite{salimans2022progressive,meng2023distillation} and quantization~\cite{shang2023post,li2023q}, which transfer knowledge to smaller models or reduce numerical precision, respectively. However, these methods often introduce notable trade-offs. Large reductions in sampling steps can compromise generation quality and diversity. Distillation requires substantial computational resources and may be limited by teacher model performance, while extreme quantization can lead to noticeable degradation.



\textbf{Diffusion Model Caching.}
Cache-based acceleration methods have emerged as a promising training-free paradigm that exploits temporal redundancy in the denoising process. 
DeepCache~\cite{ma2024deepcache} pioneers this approach by caching and reusing high-level features across adjacent denoising timesteps. $\Delta$-DiT~\cite{chen2024delta} introduces a DiT-specific caching mechanism that stores feature differences rather than raw outputs. 
FORA~\cite{selvaraju2024fora} implements a static caching strategy with a fixed interval parameter. 
More sophisticated approaches include PAB~\cite{zhao2024real}, which applies different broadcast ranges for spatial, temporal, and cross-modal attention mechanisms. 
TeaCache~\cite{liu2025timestep} intelligently determines when to use cached results based on input similarity. 
SmoothCache~\cite{liu2025smoothcache} analyzes layer-wise representation errors from calibration data to adaptively determine caching intensity. 
ProfilingDiT~\cite{ma2025model} distinguishes foreground-focused and background-focused transformer blocks, selectively caching static background features.

While these direct feature reuse methods show impressive acceleration, 
they rely on features from previous timesteps that may not perfectly align with the current denoising stage, causing accumulated errors that manifest as artifacts or quality loss.
An alternative approach involves predicting future timestep features rather than directly reusing cached results. 
TaylorSeer~\cite{liu2025reusing} employs Taylor series expansion to forecast features at future timesteps. 
AB-Cache~\cite{yu2025ab} applies an Adams-Bashforth predictor to forecast cached features but lacks a corrector mechanism and uses a fixed skip interval, offering no adaptive control over prediction accuracy. 
ABM-Solver~\cite{ma2025adams} operates at the ODE solver level, replacing standard solvers (e.g., Euler, DDIM) with Adams-Bashforth-Moulton integration for the diffusion sampling process; it is designed primarily for inversion and editing tasks rather than caching-based acceleration.
Yet these prediction-based methods adopt fixed skip intervals without adaptive control, leading to error accumulation when feature dynamics vary significantly across the diffusion trajectory.

\section{Method}
We present PrediT, a training-free acceleration framework for DiTs that formulates feature prediction as a linear multistep problem. We first analyze the limitations of existing reuse-based methods and the smoothness property of diffusion trajectories, then introduce our predictor-corrector scheme with dynamic step modulation.

\subsection{Preliminaries}

\textbf{Diffusion Models.}
DMs generate data by learning to reverse a gradual noising process. Given clean data $x_0$, a forward process progressively adds Gaussian noise over timesteps $t \in [0, 1]$ to produce $x_t$. The reverse process starts from pure noise $x_1 \sim \mathcal{N}(0, I)$ and iteratively denoises to recover $x_0$. It can be formulated as solving an Ordinary Differential Equation (ODE)~\cite{song2020denoising,song2020score}:
\begin{equation}
\frac{dx}{dt} = f_\theta(x_t, t),
\label{eq:ode}
\end{equation}
where $f_\theta$ is the model's velocity prediction~\cite{lipman2022flow}. The solution is obtained by numerical integration from $t=1$ to $t=0$.

\textbf{Diffusion Transformers.}
Diffusion Transformers (DiT)~\cite{peebles2023scalable} replace the U-Net backbone with a transformer architecture. The latent $z_t$ is patchified into tokens, processed through self-attention and MLP layers, and unpatchified back to latent dimensions. Each DiT block applies Adaptive Layer Normalization (AdaLN) conditioned on timestep $t$:
\begin{equation}
h' = \text{Attn}(\text{AdaLN}(h)), \quad h'' = \text{MLP}(\text{AdaLN}(h')).
\label{eq:dit_block}
\end{equation}
While DiT achieves superior generation quality, the quadratic attention cost combined with dozens of iterative steps makes inference computationally expensive.

\textbf{Feature Reuse.}
Current caching methods accelerate inference by reusing previous model outputs~\cite{ma2024deepcache,chen2024delta}. Given the current state $x_n$ at timestep $t_n$, these methods approximate the model output as:
\begin{equation}
\hat{f}_n \approx f_{n-k},
\label{eq:reuse}
\end{equation}
where $f_{n-k}$ is a cached output from $k$ steps earlier. This zero-order approximation assumes $f$ remains constant, incurring a local truncation error (LTE) of $\mathcal{O}(\Delta t)$. This error accumulates over consecutive skips, causing latent drift and visual artifacts.

\textbf{Linear Multistep Methods.}
Linear multistep methods~\cite{hairer1993solving,butcher2016numerical} are classical numerical techniques for solving ODEs that use multiple previous function values to achieve higher-order accuracy. Starting from the integral form of \cref{eq:ode}:
\begin{equation}
x_{n+1} = x_n + \int_{t_n}^{t_{n+1}} f(x(s), s) \, ds,
\label{eq:integral}
\end{equation}
these methods approximate the integrand by a polynomial interpolating past values $\{f_n, f_{n-1}, \ldots\}$, achieving LTE of $\mathcal{O}(\Delta t^{k+1})$ for order $k$. This provides a principled framework for predicting model outputs without explicit derivative estimation~\cite{lu2022dpm,lu2025dpm}.

\begin{figure*}[t]
  \centering
  \includegraphics[width=\linewidth]{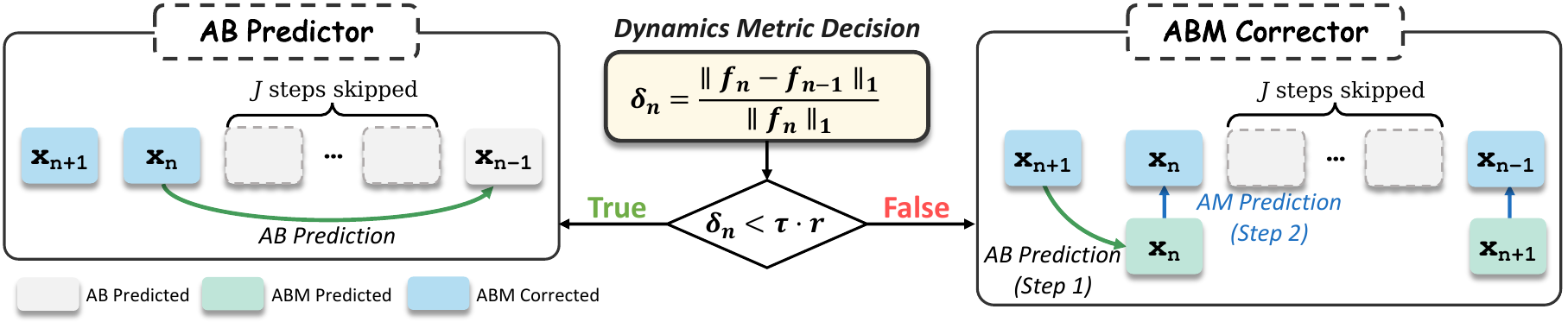}
  \caption{Overview of the PrediT framework. PrediT formulates feature prediction as a linear multistep problem, using historical model outputs to extrapolate future features and reduce redundant model calls. The threshold $\tau$ and correction ratio $r$ control the selection between ABM for high-dynamics and AB for low-dynamics regions.}
  \Description{A pipeline diagram of the PrediT framework showing historical model outputs feeding an Adams-Bashforth predictor, an Adams-Moulton corrector for high-dynamics regions, and dynamic step modulation to control skip intervals.}
  \label{fig:overview}
 \end{figure*}

\subsection{PrediT Framework}

As illustrated in \cref{fig:overview}, PrediT combines an Adams-Bashforth predictor with an Adams-Moulton corrector under dynamic step modulation.

\noindent\textbf{Adams-Bashforth Predictor.} The Adams-Bashforth (AB) method is explicit. It extrapolates from historical function values~\cite{liu2022pseudo}. For order $k$, the method takes the form:
\begin{equation}
x_{n+1} = x_n + \Delta t \sum_{j=0}^{k-1} \beta_j f_{n-j},
\label{eq:ab}
\end{equation}
where $\beta_j$ are coefficients derived from Lagrange interpolation. Detailed coefficients are provided in the supplementary material. For the second-order case (AB2), we have:
\begin{equation}
x_{n+1} = x_n + \frac{\Delta t}{2}(3f_n - f_{n-1}),
\label{eq:ab2}
\end{equation}
which achieves LTE of $\mathcal{O}(\Delta t^3)$, compared to $\mathcal{O}(\Delta t^2)$ for Euler. Crucially, AB directly combines historical function values without explicit derivative estimation, avoiding the instability of finite-difference approaches.

\textbf{Adams-Moulton Corrector.}
While AB is efficient, its explicit nature can accumulate errors in high-dynamics regions. The Adams-Moulton (AM) method is an implicit corrector that includes the future value $f_{n+1}$:
\begin{equation}
x_{n+1} = x_n + \Delta t \sum_{j=-1}^{k-1} \gamma_j f_{n-j},
\label{eq:am}
\end{equation}
where $\gamma_{-1}$ corresponds to $f_{n+1}$. For the second-order case (AM2):
\begin{equation}
x_{n+1} = x_n + \frac{\Delta t}{12}(5f_{n+1} + 8f_n - f_{n-1}).
\label{eq:am2}
\end{equation}
Combining AB as predictor and AM as corrector yields the Adams-Bashforth-Moulton (ABM) scheme~\cite{karras2022elucidating}, which first predicts $\tilde{x}_{n+1}$ using AB, then evaluates $\tilde{f}_{n+1} = f_\theta(\tilde{x}_{n+1}, t_{n+1})$, and finally corrects using AM. While ABM achieves LTE of $\mathcal{O}(\Delta t^4)$ and improved stability, it requires one extra model call compared to AB alone. As shown in \cref{fig:combined_analysis}(b), AB significantly reduces approximation error compared to naive reuse, and ABM further improves accuracy with lower accumulated drift. 

To balance accuracy and efficiency, PrediT dynamically selects between AB and ABM based on the relative feature change rate:
\begin{equation}
\delta_n = \frac{\|f_n - f_{n-1}\|_1}{\|f_n\|_1 + \epsilon},
\label{eq:delta}
\end{equation}
where $\epsilon$ is a small constant for numerical stability. ABM is applied in high-dynamics regions where accuracy is critical, while AB is used in smooth regions to maximize speedup.

\textbf{Why ABM.}
We adopt the Adams-Bashforth-Moulton (ABM) family over alternatives such as Taylor expansion and BDF for three reasons. First, the AB predictor extrapolates directly from historical function values $\{f_n, f_{n-1}, \ldots\}$, avoiding explicit derivative estimation and the instability of finite-difference Taylor methods~\cite{liu2025reusing}. Second, the AM corrector incorporates the newly evaluated $\tilde{f}_{n+1}$ to suppress drift in high-dynamics regions under adaptive skipping. Third, AB and AM use interpolation-determined coefficients with no extra hyperparameters, giving a practical balance of efficiency and stability for diffusion trajectories; this is consistent with the strong degradation of Taylor-based methods at longer trajectories in \cref{tab:dit}.

\textbf{Dynamic Step Modulation.}
As shown in \cref{fig:sensitivity_case1}, the model's feature dynamics vary significantly along the diffusion trajectory, with initial and final phases exhibiting rapid changes. \cref{fig:sensitivity_case2} further reveals that skipping in high-dynamics regions incurs substantially larger feature deviation than in low-dynamics regions. Notably, even skipping multiple steps in smooth regions produces less error than a single skip in high-dynamics regions. Therefore, a fixed skip schedule cannot adapt to this variation~\cite{karras2024analyzing}. We compute the adaptive skip interval from the dynamics metric $\delta_n$:
\begin{equation}
J = \left\lfloor \frac{\tau}{(\delta_n + \epsilon)^{1/(p+1)}} \right\rfloor,
\label{eq:jump}
\end{equation}
where $\tau$ is the threshold parameter and $p$ is the sensitivity exponent. The decision logic selects ABM without skipping when $\delta_n \geq \tau$, ABM with limited skipping when $\tau \cdot r \leq \delta_n < \tau$, and AB only with large skipping when $\delta_n < \tau \cdot r$, where $r$ is the correction ratio. This ensures model calls focus on high-dynamics regions while maximizing skips in smooth regions.

\textbf{Robustness of DSM.}
DSM assumes that the dynamics metric $\delta_n$ can predict feature behavior over the next $J$ steps. This assumption is supported by \cref{fig:pca}, where the feature change rate remains smooth within each phase of the diffusion trajectory. If the dynamics shift during a skipped interval, the next computed step produces a large $\delta_n$, which triggers ABM correction and resets the drift and skip interval. In our experiments, this first-order metric with ABM correction is sufficient to maintain robust performance.

\textbf{Error Analysis.}
The total approximation error consists of discretization error from numerical integration, prediction error from polynomial extrapolation during skips, and accumulated drift from consecutive skips. Our framework mitigates these by using higher-order AB/AM methods to reduce discretization error, adaptive skip interval to limit prediction error in curved regions, and ABM correction to reset accumulated drift when $\delta_n$ exceeds the threshold. Detailed error bounds and the complete algorithm are provided in the supplementary material.

\begin{table*}[t]
  \centering
  \caption{Comparison of visual quality and efficiency in text-to-image generation on FLUX.}
  \label{tab:flux}
  \resizebox{\textwidth}{!}{%
  \begin{tabular}{c|cccc|ccc|ccc}
  \toprule
  \multirow{2}{*}{\textbf{Method}}              & \multicolumn{4}{c|}{\textbf{Acceleration}}                         & \multicolumn{6}{c}{\textbf{Visual Quality}} \\ \cline{2-11}
  &
   \textbf{Latency(s)$\downarrow$} &
   \multicolumn{1}{c|}{\textbf{Speedup$\uparrow$}} &
   \textbf{FLOP(T)$\downarrow$} &
   \textbf{Speedup$\uparrow$} &
   \textbf{Img.$\uparrow$} &
   \textbf{CLIP$\uparrow$} &
   \textbf{Aes.$\uparrow$} &
   \textbf{LPIPS$\downarrow$} &
   \textbf{SSIM$\uparrow$} &
   \textbf{PSNR$\uparrow$} \\ \midrule
  \rowcolor{gray!20} \textbf{Original: 50 Steps} & 23.71 & \multicolumn{1}{c|}{1.00$\times$} & 3719.50 & 1.00$\times$ & 0.9767 & 31.94 & 6.7448 & \cellcolor{gray!20}0.0000 & \cellcolor{gray!20}1.000 & \cellcolor{gray!20}$\infty$  \\ 
  $\Delta$-DiT($\mathcal{N}=2$)       & 15.86 & \multicolumn{1}{c|}{1.50$\times$} & 1859.75 & 2.00$\times$ & 0.9357 & 32.02 & 6.5344 & 0.3074 & 0.6783 & 18.63 \\ 
  TeaCache($\tau=0.4$)         & 12.31 & \multicolumn{1}{c|}{1.92$\times$} & 1438.92 & 2.58$\times$ & 0.8710 & 31.79 & 6.4400 & 0.6960 & 0.4175 & 11.54 \\
  ABM-Solver($\tau=0.1$)          & 13.05 & \multicolumn{1}{c|}{1.82$\times$} & 952.22  & 3.90$\times$ & 0.9378 & \underline{32.08} & 6.7260 & 0.4762 & 0.6014 & 17.72 \\\hline
  \rowcolor{gray!20} \textbf{Original: 25 Steps} & 11.58 & \multicolumn{1}{c|}{2.05$\times$} & 1859.75 & 2.00$\times$ & 0.9467 & 31.64 & \underline{6.7307} & 0.2970 & 0.6818 & 18.85 \\
  $\Delta$-DiT($\mathcal{N}=3$)       & 11.50 & \multicolumn{1}{c|}{2.06$\times$} & 1264.63 & 2.94$\times$ & 0.9385 & 31.87 & 6.4825 & 0.3353 & 0.6805 & 18.63 \\
  TeaCache($\tau=0.8$)         & 10.46 & \multicolumn{1}{c|}{2.26$\times$} & 899.32  & 4.13$\times$ & 0.7933 & 31.67 & 6.4048 & 0.7205 & 0.4236 & 11.80 \\ 
  FORA($\mathcal{N}=3$)               & 8.73  & \multicolumn{1}{c|}{2.71$\times$} & 1264.63 & 2.94$\times$ & 0.9715 & 32.06 & 6.5433 & 0.3942 & 0.6104 & 16.74 \\
  AB-Cache($\mathcal{N}=3$)           & 8.74  & \multicolumn{1}{c|}{2.71$\times$} & 1339.06 & 2.78$\times$ & 0.9357 & 31.85 & 6.6625 & 0.2987 & 0.6752 & \underline{19.42} \\
  TaylorSeer($\mathcal{O}=2$, $\mathcal{N}=5$)         & 6.52  & \multicolumn{1}{c|}{3.64$\times$} & 892.68  & 4.17$\times$ & 0.9676 & 31.94 & \textbf{6.7439} & 0.3848 & 0.6174 & 18.18 \\\hline
  FORA($\mathcal{N}=4$)               & 6.90  & \multicolumn{1}{c|}{3.44$\times$} & 967.07  & 3.85$\times$ & 0.9618 & 32.07 & 6.5243 & 0.4556 & 0.5716 & 15.69 \\
  AB-Cache($\mathcal{N}=4$)           & 6.49  & \multicolumn{1}{c|}{3.65$\times$} & 1026.76 & 3.62$\times$ & 0.9308 & 31.81 & 6.6681 & 0.4084 & 0.6167 & 17.88 \\
  TaylorSeer($\mathcal{O}=2$, $\mathcal{N}=6$)         & 6.02  & \multicolumn{1}{c|}{3.94$\times$} & 743.90  & 5.00$\times$ & \underline{0.9761} & 31.89 & 6.7225 & 0.5292 & 0.5180 & 16.67 \\
  \rowcolor{gray!20} \textbf{PrediT($\mathcal{O}=2$, $\tau=2$)}            & 5.54  & \multicolumn{1}{c|}{4.28$\times$} & 743.92  & 5.00$\times$ & \textbf{0.9773} & \textbf{32.19} & \underline{6.7307} & \textbf{0.2215} & \textbf{0.7354} & \textbf{20.84} \\
  \rowcolor{gray!20} \textbf{PrediT($\mathcal{O}=2$, $\tau=2.5$)}          & \textbf{4.28}  & \multicolumn{1}{c|}{\textbf{5.54$\times$}} & \textbf{639.77}  & \textbf{5.81$\times$} & 0.9506 & 32.07 & 6.7174 & \underline{0.2842} & \underline{0.6884} & 19.09 \\ \bottomrule
  \end{tabular}%
  }
  \raggedright
  \footnotesize
  $\dagger$ \textbf{Bold} indicates the best result, \underline{underline} indicates the second best. The same applies to the following tables.
  \end{table*}

\afterpage{%
\begin{figure}[t]
 \centering
 \begin{subfigure}[b]{0.49\columnwidth}
   \centering
   \includegraphics[width=\linewidth]{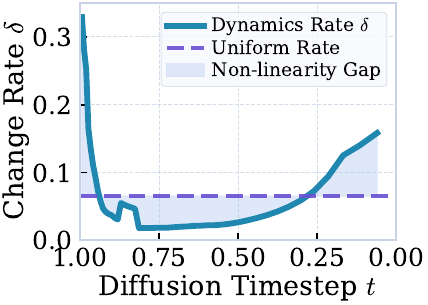}
   \caption{Non-uniform Change Rate}
   \label{fig:sensitivity_case1}
 \end{subfigure}
 \hfill
 \begin{subfigure}[b]{0.49\columnwidth}
   \centering
   \includegraphics[width=\linewidth]{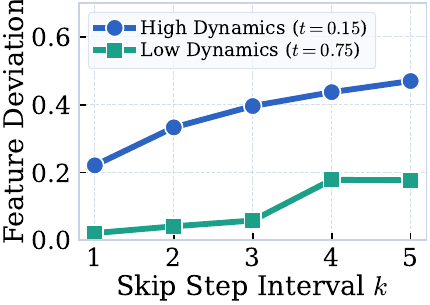}
   \caption{Skip Sensitivity}
   \label{fig:sensitivity_case2}
 \end{subfigure}
 \caption{Feature dynamics vary across the diffusion process. (a) Initial and final phases exhibit rapid changes. (b) Approximation error grows with skip interval, especially in high-dynamics regions.}
\Description{A two-panel figure showing that feature dynamics vary across timesteps and that approximation error increases as the skip interval grows, especially in high-dynamics regions.}
\label{fig:sensitivity_single_col}
\end{figure}

  \begin{figure}[t]
    \centering
    \includegraphics[width=\columnwidth]{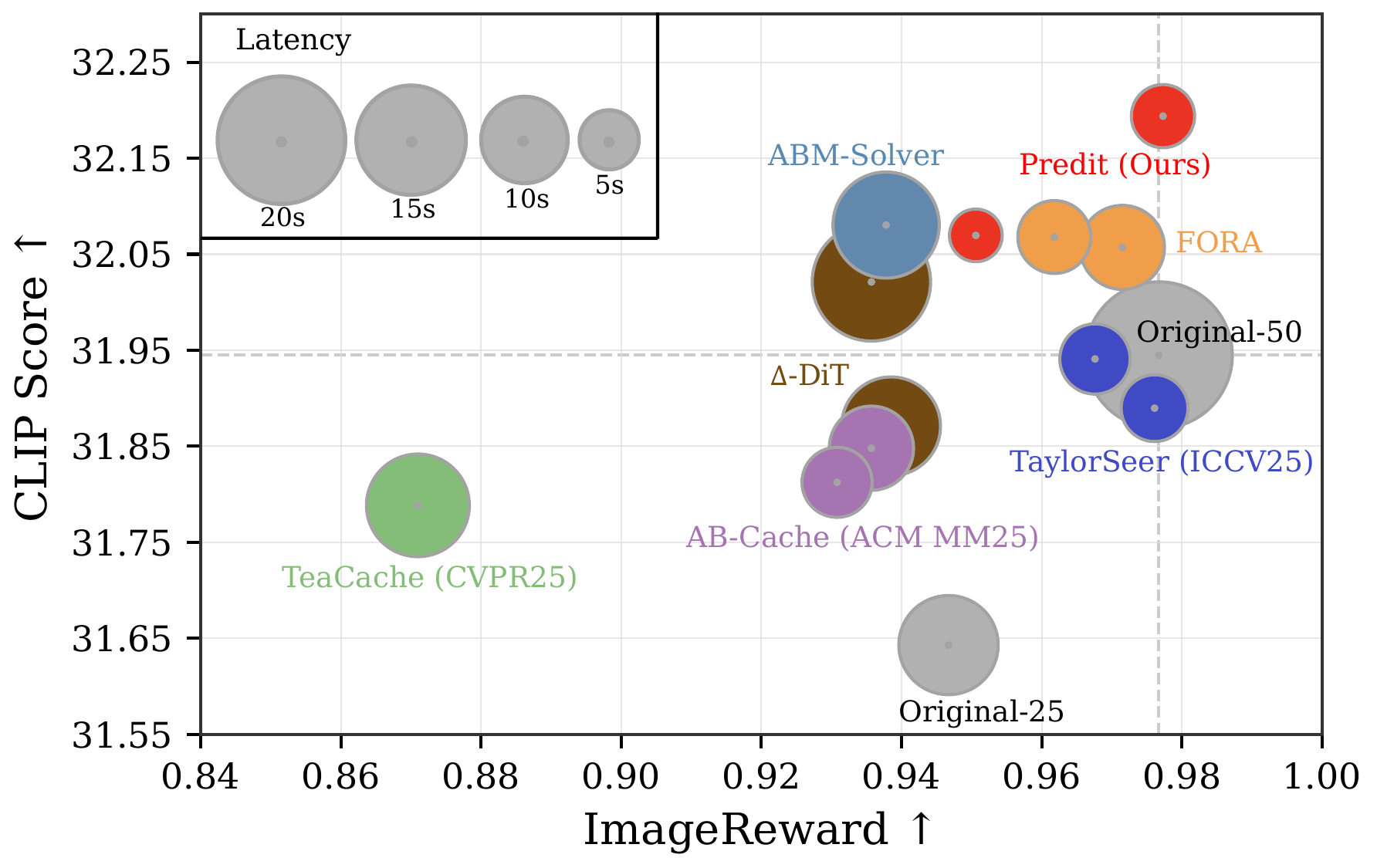}
    \caption{Bubble chart visualization of quality-efficiency trade-off on FLUX. Each bubble represents a method: position encodes ImageReward ($x$) and CLIP Score ($y$); size is proportional to inference latency (smaller is faster). PrediT (red) achieves the best quality with the lowest latency.}
    \Description{A bubble chart comparing FLUX acceleration methods, with ImageReward on the horizontal axis, CLIP Score on the vertical axis, and bubble size representing inference latency.}
    \label{fig:flux_bubble}
   \end{figure}
}

\section{Experiments}
\subsection{Experimental Settings}
\textbf{Models, Datasets, and Solvers.}
We evaluate PrediT on three representative DiT-based models. For text-to-image generation, we use FLUX.1-dev~\cite{flux2024} with 50 sampling steps and prompts from DiffusionDB~\cite{wangDiffusionDBLargescalePrompt2022}. For text-to-video generation, we use HunyuanVideo~\cite{kong2024hunyuanvideo} with the Rectified Flow~\cite{liu2022flow} solver and prompts from the VBench gallery~\cite{zhang2024evaluationagent}. For class-to-image generation, we use DiT-XL/2~\cite{peebles2023scalable} with the DDIM solver~\cite{song2020denoising} on ImageNet~\cite{deng2009imagenet} at 256$\times$256 resolution.
For the diagnostic analyses in \cref{fig:combined_analysis,fig:sensitivity_single_col}, we use the same FLUX.1-dev setup with 50 sampling steps and a subset of 20 DiffusionDB prompts. 

\textbf{Evaluation Metrics.}
For text-to-image, we report ImageReward~\cite{xu2023imagereward} to measure human preference alignment, CLIP Score~\cite{radford2021learning} for text-image consistency, and Aesthetic Score for visual appeal. For text-to-video, we use VBench~\cite{zhang2024evaluationagent} to summarize video generation performance across multiple dimensions, where Quality measures perceptual visual fidelity and Semantic measures text-video alignment. For both text-to-image and text-to-video, we use LPIPS~\cite{zhang2018unreasonable}, SSIM~\cite{wang2004image}, and PSNR to measure fidelity relative to the original outputs. For class-to-image, we report FID and sFID~\cite{heusel2017gans} for distribution similarity and Inception Score (IS)~\cite{salimans2016improved} for sample quality and diversity.
Detailed metric definitions are provided in the supplementary material.

\textbf{Implementation Details.}
All experiments are conducted with FlashAttention~\cite{dao2022flashattention} enabled. FLUX.1 and HunyuanVideo are evaluated on a single NVIDIA A800 80GB GPU, while DiT-XL/2 uses an NVIDIA RTX 4090 24GB GPU. For FLUX.1, we generate 1024$\times$1024 images. For HunyuanVideo, we test two settings: 544p$\times$960p with 17 frames and 480p$\times$640p with 45 frames. For DiT-XL/2, we generate 50K 256$\times$256 images on ImageNet for FID evaluation. All baselines share the same solver, number of denoising steps, and CFG scale, unless otherwise specified by the original method.

\subsection{Text-to-Image Generation}

\textbf{Quantitative Results.}
We evaluate PrediT on FLUX.1 for text-to-image generation and compare with representative acceleration methods in \cref{tab:flux}. The last three columns further show that PrediT preserves high fidelity to the original outputs, achieving the lowest LPIPS together with the highest SSIM and PSNR among the accelerated methods. At comparable acceleration ratios, PrediT matches or improves generation quality across all metrics. Specifically, PrediT achieves superior speedup while attaining the highest ImageReward and CLIP Score, even surpassing the original 50-step baseline. When pushing to more aggressive acceleration, PrediT reaches up to $5.54\times$ speedup with only marginal quality degradation. In contrast, reuse-based methods such as $\Delta$-DiT and FORA suffer from noticeable quality drops at lower speedups. TeaCache achieves only $2.26\times$ speedup but yields a lower ImageReward, significantly below the original. Prediction-based methods like TaylorSeer show improved quality retention but still fall short of PrediT at equivalent accelerations.

To provide a more intuitive comparison, we visualize the trade-off between generation quality and inference efficiency in \cref{fig:flux_bubble}. Each bubble represents a method, where the position indicates ImageReward (x-axis) and CLIP Score (y-axis), and the bubble size is proportional to inference latency—smaller bubbles correspond to faster methods. As shown, PrediT occupies the upper-right region with the smallest bubbles, demonstrating that it simultaneously achieves the best quality metrics and the lowest latency among all compared methods. The dashed lines indicate the Original 50-step baseline performance for reference.

\begin{figure}[t]
  \centering
  \includegraphics[width=\linewidth]{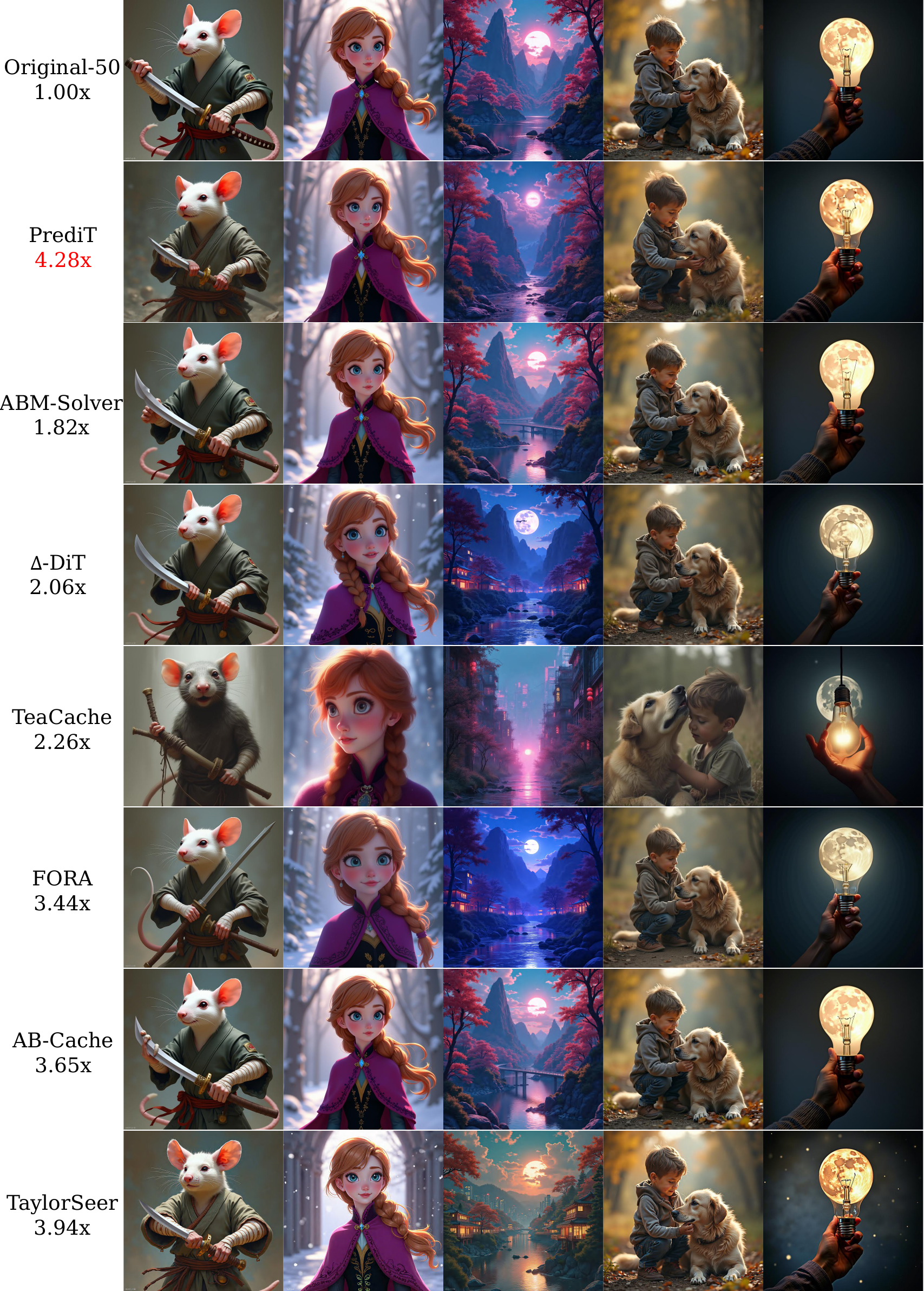}
  \caption{Visual quality comparison of different acceleration methods on FLUX. PrediT achieves higher acceleration while maintaining comparable visual quality. Zoom in for more details.}
  \Description{A qualitative comparison grid of FLUX text-to-image results across several prompts, contrasting the original model, baseline acceleration methods, and PrediT.}
  \label{fig:flux}
 \end{figure}

\textbf{Visual Comparison.}
\cref{fig:flux} presents qualitative comparisons on FLUX.1, and additional results with more prompts are provided in the supplementary material. Other acceleration methods exhibit visible artifacts including blurriness and loss of fine details, particularly in regions with complex textures. PrediT maintains visual fidelity comparable to the original model while achieving significantly higher acceleration, demonstrating the effectiveness of our principled prediction approach over naive feature reuse.

\textbf{Ablation Study.}
To validate the contribution of each component in PrediT, we conduct an ablation study on FLUX.1 as shown in \cref{tab:ablation}. The upper part of the table analyzes different module combinations. Using only AB or only ABM does not yield the best overall trade-off, even when combined with DSM. AB-based variants can provide higher acceleration but suffer from weaker quality, while ABM-based variants improve quality but are less efficient. In contrast, PrediT-DSM, which combines AB, ABM, and DSM in a unified framework, achieves the best balance between acceleration and generation quality. The lower part further compares PrediT-DSM with static skipping baselines using fixed skip intervals $J \in \{3,4,5\}$. Although larger fixed $J$ improves speedup, it leads to less stable quality. PrediT-DSM outperforms these static baselines in the overall quality-efficiency trade-off, showing the benefit of adaptive step modulation over fixed skipping.

\begin{table}[t]
 \centering
 \caption{Ablation study of predictor-corrector components and DSM, with comparison to static skipping baselines on FLUX.1.}
 \label{tab:ablation}
 \resizebox{\columnwidth}{!}{%
 \begin{tabular}{c|cc|cc}
 \toprule
 \textbf{Method} & \multicolumn{2}{c|}{\textbf{Acceleration}} & \textbf{Image} & \textbf{CLIP} \\ \cline{2-3}
 \textbf{FLUX.1} & \textbf{Latency(s) $\downarrow$} & \textbf{Speedup $\uparrow$} & \textbf{Reward $\uparrow$} & \textbf{Score $\uparrow$} \\ \midrule
 \rowcolor{gray!20} \textbf{Original: 50 Steps}        & 23.71            & 1.00$\times$            & 0.9767         & 31.9448       \\
 \textbf{Only AB \textit{w/o} \texttt{DSM}}  & 6.49             & 3.65$\times$            & 0.9308         & 31.8125       \\
 \textbf{Only ABM \textit{w/o} \texttt{DSM}} & 11.89            & 1.99$\times$            & 0.9318         & 31.8634       \\
 \textbf{Only AB + \texttt{DSM}}             & 3.65             & 6.49$\times$            & 0.9374         & 31.8534       \\
 \textbf{Only ABM + \texttt{DSM}}            & 6.15             & 3.86$\times$            & 0.9671         & 32.1778       \\ \hline
 \textbf{Static Skip (J=3)} & 8.64 & 2.73$\times$ & 0.9796 & 31.3496 \\
 \textbf{Static Skip (J=4)} & 6.79 & 3.47$\times$ & 0.9753 & 31.3597 \\
 \textbf{Static Skip (J=5)} & 5.37 & 4.39$\times$ & 0.9606 & 31.3010 \\
 \rowcolor{gray!20} \textbf{PrediT-DSM}          & 5.54             & 4.28$\times$            & 0.9773         & 32.1938       \\ \bottomrule
 \end{tabular}%
 }
 \end{table}
 
 \begin{table*}[t]
   \centering
   \caption{Comparison of visual quality and efficiency in text-to-video generation on HunyuanVideo on a single A800 GPU.}
   \label{tab:video}
   \resizebox{\textwidth}{!}{%
   \begin{tabular}{c|c|cc|ccc|ccc}
   \toprule
   \multirow{2}{*}{\textbf{Settings}} & \multirow{2}{*}{\textbf{Method}} & \multicolumn{2}{c|}{\textbf{Acceleration}} & \multicolumn{6}{c}{\textbf{Visual Quality}} \\ \cline{3-10}
    &  & \textbf{Latency(s) $\downarrow$} & \textbf{Speedup $\uparrow$} & \textbf{VBench$\uparrow$} & \textbf{Quality$\uparrow$} & \textbf{Semantic$\uparrow$} & \textbf{LPIPS$\downarrow$} & \textbf{SSIM$\uparrow$} & \textbf{PSNR$\uparrow$} \\ \midrule
   \multirow{9}{*}{\begin{tabular}[c]{@{}c@{}}544p$\times$960p, \\ 17 frames\end{tabular}} & \cellcolor{gray!20}\textbf{Original: 50 Steps} & \cellcolor{gray!20}76.71 & \cellcolor{gray!20}1.00$\times$ & \cellcolor{gray!20}82.08 & \cellcolor{gray!20}82.86 & \cellcolor{gray!20}78.98 & \cellcolor{gray!20}0.0000 & \cellcolor{gray!20}1.000 & \cellcolor{gray!20}$\infty$ \\
    & PAB($\mathcal{N}=8$) & 66.21 & 1.16$\times$ & 80.81 & 81.36 & 78.60 & 0.1433 & 0.8370 & \textbf{27.41} \\
    & TeaCache($\tau=0.14$) & 57.71 & 1.33$\times$ & 82.18 & 82.36 & 81.44 & 0.1466 & 0.8536 & 25.64 \\ \cline{2-10}
    & \textbf{Original: 25 Steps} & 40.29 & 1.90$\times$ & 81.51 & 82.10 & 78.68 & 0.1973 & 0.8196 & 23.17 \\
    & ProfilingDiT($\mathcal{N}=6$) & 44.60 & 1.72$\times$ & 80.33 & 80.02 & \underline{81.58} & 0.2610 & 0.6662 & 18.77 \\
    & TaylorSeer($\mathcal{O}=1$, $\mathcal{N}=4$) & 39.30 & 1.95$\times$ & 82.05 & 82.33 & 80.90 & 0.4021 & 0.5526 & 15.33 \\
    & ABM-Solver($\tau=0.1$) & 34.84 & 2.20$\times$ & 82.05 & 82.33 & 80.90 & 0.1632 & 0.8488 & 25.07 \\
    & AB-Cache($\mathcal{N}=3$) & 32.14 & 2.39$\times$ & \underline{82.36} & \underline{82.72} & 80.93 & 0.1587 & 0.8512 & 24.67 \\
    & \cellcolor{gray!20}\textbf{PrediT($\mathcal{O}=2$, $\tau=2$)} & \cellcolor{gray!20}27.21 & \cellcolor{gray!20}2.81$\times$ & \cellcolor{gray!20}\textbf{82.45} & \cellcolor{gray!20}\textbf{83.04} & \cellcolor{gray!20}80.10 & \cellcolor{gray!20}\textbf{0.1329} & \cellcolor{gray!20}\textbf{0.8693} & \cellcolor{gray!20}\underline{25.81} \\
    & \cellcolor{gray!20}\textbf{PrediT($\mathcal{O}=3$, $\tau=2.5$)} & \cellcolor{gray!20}\textbf{23.36} & \cellcolor{gray!20}\textbf{3.28$\times$} & \cellcolor{gray!20}\underline{82.36} & \cellcolor{gray!20}82.27 & \cellcolor{gray!20}\textbf{82.73} & \cellcolor{gray!20}\underline{0.1429} & \cellcolor{gray!20}\underline{0.8616} & \cellcolor{gray!20}25.77 \\ \midrule
   \multirow{7}{*}{\begin{tabular}[c]{@{}c@{}}480p$\times$640p, \\ 45 frames\end{tabular}} & \cellcolor{gray!20}\textbf{Original: 50 Steps} & \cellcolor{gray!20}117.43 & \cellcolor{gray!20}1.00$\times$ & \cellcolor{gray!20}80.12 & \cellcolor{gray!20}80.51 & \cellcolor{gray!20}78.54 & \cellcolor{gray!20}0.0000 & \cellcolor{gray!20}1.000 & \cellcolor{gray!20}$\infty$ \\
    & PAB($\mathcal{N}=8$) & 95.14 & 1.23$\times$ & 78.45 & 78.94 & 76.52 & 0.1712 & 0.7969 & \textbf{25.67} \\
    & TeaCache($\tau=0.14$) & 77.86 & 1.51$\times$ & 80.05 & 80.43 & 78.52 & 0.1701 & 0.8036 & 23.48 \\ \cline{2-10}
    & \textbf{Original: 25 Steps} & 62.14 & 1.89$\times$ & 79.76 & 80.39 & 77.25 & 0.2221 & 0.7599 & 20.95 \\
    & ABM-Solver($\tau=0.1$) & 50.18 & 2.34$\times$ & 79.71 & 79.86 & \underline{79.11} & 0.1650 & 0.8186 & 24.02 \\
    & AB-Cache($\mathcal{N}=3$) & 49.29 & 2.38$\times$ & 79.89 & \underline{80.66} & 76.85 & 0.1801 & 0.7995 & 22.68 \\
    & \cellcolor{gray!20}\textbf{PrediT($\mathcal{O}=2$, $\tau=2$)} & \cellcolor{gray!20}41.29 & \cellcolor{gray!20}2.84$\times$ & \cellcolor{gray!20}\textbf{80.14} & \cellcolor{gray!20}80.35 & \cellcolor{gray!20}\textbf{79.32} & \cellcolor{gray!20}\textbf{0.1401} & \cellcolor{gray!20}\textbf{0.8345} & \cellcolor{gray!20}\underline{24.25} \\
    & \cellcolor{gray!20}\textbf{PrediT($\mathcal{O}=3$, $\tau=2.5$)} & \cellcolor{gray!20}\textbf{36.14} & \cellcolor{gray!20}\textbf{3.24$\times$} & \cellcolor{gray!20}\underline{80.08} & \cellcolor{gray!20}\textbf{80.80} & \cellcolor{gray!20}77.25 & \cellcolor{gray!20}\underline{0.1513} & \cellcolor{gray!20}\underline{0.8280} & \cellcolor{gray!20}24.06 \\ \bottomrule
   \end{tabular}%
   }
   \raggedright
   \footnotesize
   $\dagger$ ProfilingDiT and TaylorSeer can only generate up to 17 frames at 544p$\times$960p resolution on a single A800 GPU, and run OOM at 480p$\times$640p with 45 frames.
 \end{table*}
 
\subsection{Text-to-Video Generation}

\textbf{Quantitative Results.}
We extend our evaluation to text-to-video generation on HunyuanVideo, testing two resolution settings as shown in \cref{tab:video}. At 544p$\times$960p with 17 frames, PrediT achieves $3.28\times$ speedup while maintaining the highest VBench score and best visual metrics. Compared to other methods, PAB and TeaCache provide limited speedups with 50 sampling steps. Prediction-based methods face challenges in video generation; for example, TaylorSeer achieves comparable speedup but with degraded visual quality, indicating accumulated prediction errors across video frames. AB-Cache shows better quality retention but PrediT still outperforms it in both acceleration and visual fidelity. At 480p$\times$640p with 45 frames, the computational challenge increases, yet PrediT maintains consistent acceleration with the best VBench score. Supplementary memory measurements further show that PrediT adds only 0.1\%--2.1\% GPU memory over the original model across five HunyuanVideo settings, while TeaCache requires 7.5\%--12.1\% extra memory. PAB already increases memory by 27.6\%--80.2\% in the few settings where it runs, and ProfilingDiT and TaylorSeer encounter out-of-memory (OOM) errors on most larger settings. These results highlight that PrediT remains practical for longer sequences and higher resolutions rather than merely accelerating easy cases.

\begin{figure}[t]
 \centering
 \begin{subfigure}[b]{0.49\columnwidth}
   \centering
   \includegraphics[width=\linewidth]{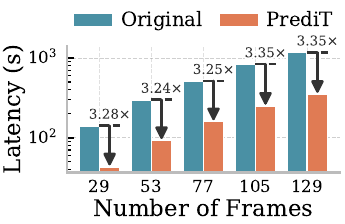}
   \caption{Latency scaling at 544P}
   \label{fig:acceleration_544x960}
 \end{subfigure}
 \hfill
 \begin{subfigure}[b]{0.49\columnwidth}
   \centering
   \includegraphics[width=\linewidth]{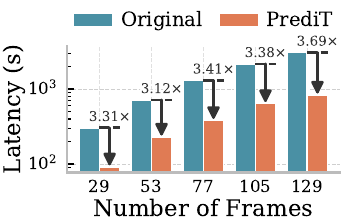}
   \caption{Latency scaling at 720P}
   \label{fig:acceleration_720x1280}
 \end{subfigure}
 \caption{Latency comparison between the original method and PrediT on HunyuanVideo at 544P and 720P resolutions across different frame counts.}
 \Description{Two line plots comparing latency versus frame count for the original HunyuanVideo pipeline and PrediT at 544P and 720P resolutions.}
 \label{fig:acceleration}
\end{figure}

\textbf{Scaling Analysis.}
\cref{fig:acceleration} illustrates the latency scaling behavior across different frame counts at 544P and 720P resolutions. The original method exhibits near-linear latency growth with increasing frames, becoming prohibitive for longer video generation. PrediT consistently reduces latency across all configurations, with the acceleration ratio remaining stable as frame count increases. This demonstrates that our dynamic step modulation adapts effectively to the temporal structure of video diffusion, maintaining efficient inference regardless of video length.

\textbf{Visual Comparison.}
\cref{fig:hy} presents visual comparisons on HunyuanVideo; additional visual results are provided in the supplementary material. Other acceleration methods suffer from noticeable quality degradation, including blurriness and loss of fine details in both spatial and temporal dimensions. PrediT preserves the visual quality of the original model, with sharp details and consistent motion across frames, validating the effectiveness of our predictor-corrector scheme.

\begin{figure*}[t]
 \centering
 \includegraphics[width=\linewidth]{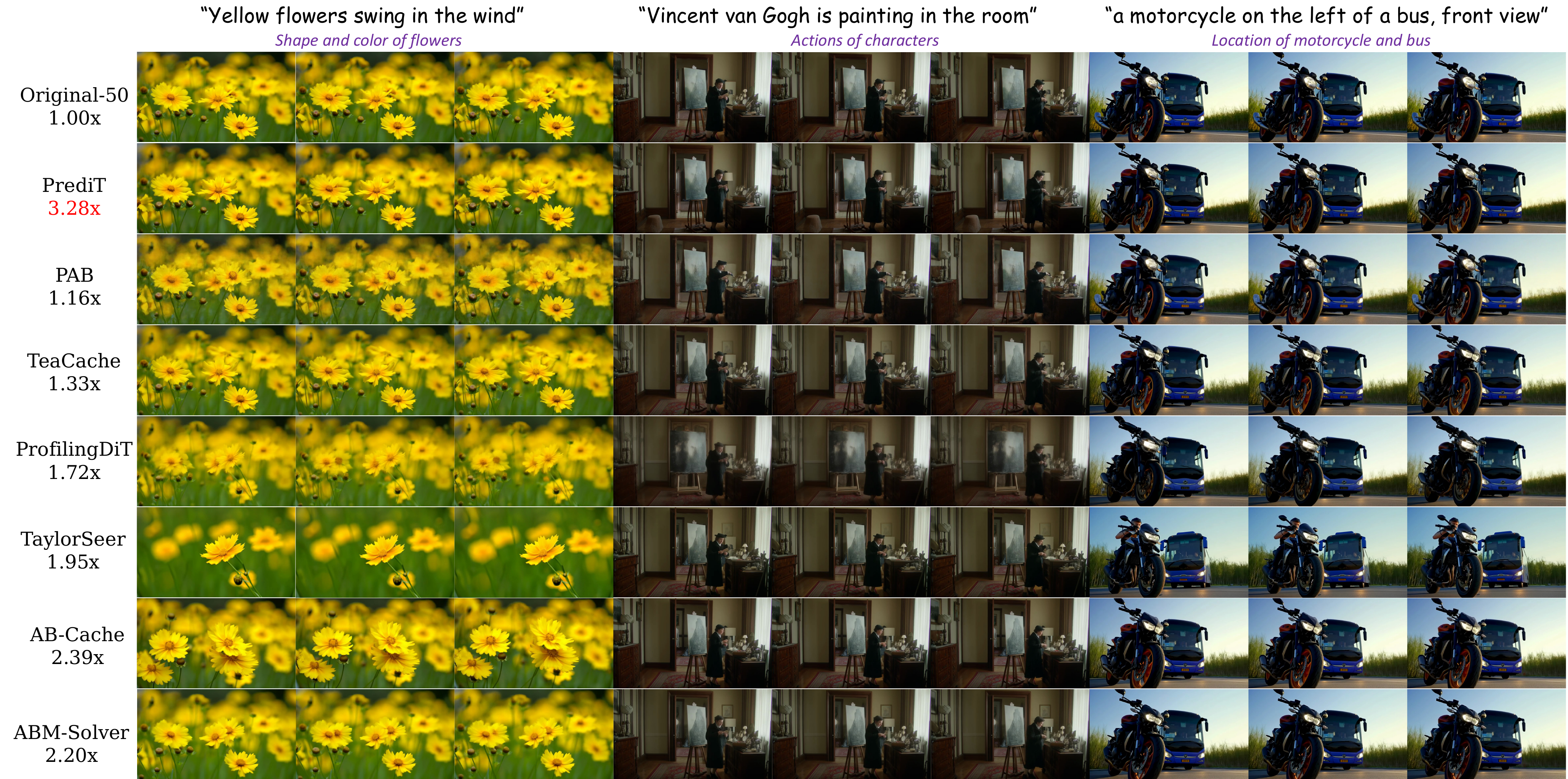}
 \caption{Visual quality comparison of different acceleration methods on HunyuanVideo. Other acceleration methods suffer from blurriness or loss of details, while PrediT maintains comparable visual quality with higher acceleration. Zoom in for more details.}
 \Description{A qualitative comparison grid of video frames from HunyuanVideo, showing the original method, baseline accelerators, and PrediT across several prompts.}
\label{fig:hy}
\end{figure*}

\afterpage{\afterpage{%
\begin{table}[!t]
 \centering
 \caption{Comparison of visual quality and efficiency in class-to-image generation on DiT-XL/2.}
 \label{tab:dit}
 \resizebox{\columnwidth}{!}{%
 \begin{tabular}{c|c|cc|ccc}
 \toprule
 \textbf{Method} & \textbf{Steps} & \textbf{Latency(s) $\downarrow$} & \textbf{Speedup $\uparrow$} & \textbf{FID$\downarrow$} & \textbf{sFID$\downarrow$} & \textbf{IS$\uparrow$} \\ \midrule
 \rowcolor{gray!20}\textbf{DDIM} & 50 & 0.569 & 1.00$\times$ & 2.28 & 4.25 & 242.26 \\
 L2C$_{0.5}$ & 50 & 0.455 & 1.25$\times$ & 2.39 & \textbf{4.40} & 235.75 \\
 FORA$_{2}$ & 50 & 0.331 & 1.72$\times$ & 2.73 & 4.90 & 234.57 \\
 FORA$_{3}$ & 50 & 0.280 & 2.03$\times$ & 3.52 & 6.39 & 226.18 \\
 TaylorSeer$^{3}_{3}$ & 50 & 0.409 & 1.39$\times$ & 2.41 & 4.69 & 235.83 \\
 \rowcolor{gray!20}\textbf{PrediT}$^{4}_{1,0.1}$ & 50 & 0.309 & 1.84$\times$ & \textbf{2.23} & \underline{4.60} & \textbf{240.10} \\
 \rowcolor{gray!20}\textbf{PrediT}$^{4}_{1,0.3}$ & 50 & \textbf{0.268} & \textbf{2.12$\times$} & \underline{2.24} & 4.78 & \underline{239.55} \\ \midrule
 \rowcolor{gray!20}\textbf{DDIM} & 70 & 0.779 & 1.00$\times$ & 2.21 & 4.29 & 243.92 \\
 FORA$_{2}$ & 70 & 0.449 & 1.73$\times$ & 2.43 & \textbf{4.59} & \underline{244.19} \\
 FORA$_{3}$ & 70 & 0.335 & 2.33$\times$ & 2.85 & 5.40 & 241.42 \\
 TaylorSeer$^{3}_{3}$ & 70 & 0.553 & 1.41$\times$ & 10.12 & 16.35 & 175.62 \\
 \rowcolor{gray!20}\textbf{PrediT}$^{4}_{1,0.1}$ & 70 & 0.356 & 2.19$\times$ & \textbf{2.15} & \underline{4.61} & \textbf{245.15} \\
 \rowcolor{gray!20}\textbf{PrediT}$^{4}_{1,0.3}$ & 70 & \textbf{0.314} & \textbf{2.48$\times$} & \underline{2.24} & 4.77 & 239.54 \\ \bottomrule
 \end{tabular}%
 }
\end{table}%
\begin{figure}[!t]
 \centering
 \includegraphics[width=\linewidth]{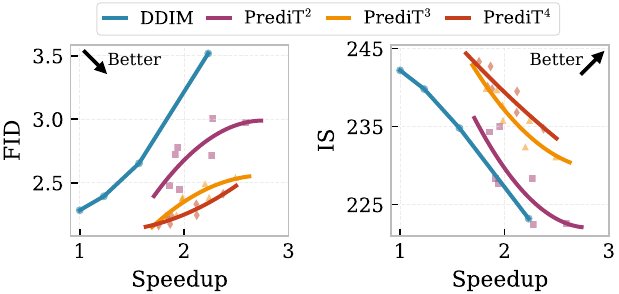}
 \caption{Quality-speedup trade-off on DiT-XL/2. PrediT achieves better quality at a higher speedup compared to the original model.}
 \Description{A trade-off plot on DiT-XL/2 comparing DDIM and PrediT configurations, showing that PrediT attains better quality at higher speedups.}
 \label{fig:dit_performance}
\end{figure}
}}

\subsection{Class-to-Image Generation}

\textbf{Quantitative Results.}
In \cref{tab:dit}, we evaluate PrediT on DiT-XL/2 for class-conditional ImageNet generation, comparing with existing acceleration methods under both 50-step and 70-step DDIM sampling. Notably, Learning-to-Cache (L2C)~\cite{ma2024learningtocache} learns to conduct caching in a dynamic manner for diffusion transformers, which is not training-free. At 50 steps, PrediT achieves $2.12\times$ speedup while improving FID from 2.28 to 2.24, demonstrating that our method can simultaneously accelerate inference and enhance generation quality. This improvement stems from the reduced discretization error of our higher-order prediction scheme.
Specifically, the standard DDIM solver uses a first-order Euler update ($\mathcal{O}(\Delta t^2)$ LTE), while PrediT's AB predictor with order $k$ achieves $\mathcal{O}(\Delta t^{k+1})$ LTE. When the predictor-corrector scheme is applied at computed steps, it effectively produces a more accurate ODE solution than the baseline solver, even though some intermediate steps are skipped. Additional ablation results in the supplementary material further support this interpretation, showing that higher predictor orders consistently yield better FID than lower orders. In comparison, FORA achieves $2.03\times$ speedup but degrades FID to 3.52, and L2C provides only $1.25\times$ acceleration. At 70 steps, PrediT maintains consistent performance with $2.48\times$ speedup and FID of 2.24. Notably, TaylorSeer exhibits severe quality collapse at 70 steps with FID degrading to 10.12, indicating numerical instability in its Taylor series extrapolation over longer sampling trajectories. PrediT avoids this issue through the stability properties of Adams methods and adaptive corrector activation.

\noindent\textbf{Quality-Speedup Trade-off.} \cref{fig:dit_performance} shows the quality-speedup trade-off between DDIM and PrediT with different predictor orders. Higher-order predictors achieve better FID and IS at the cost of slightly reduced speedup. Notably, PrediT with $\mathcal{O}=4$ surpasses the DDIM baseline in both FID and IS while still providing acceleration, demonstrating that our method can improve quality while accelerating inference. This also allows flexible parameter choices. Supplementary ablations reveal a consistent trend: increasing the predictor order from $\mathcal{O}=2$ to $\mathcal{O}=4$ steadily improves FID, while reducing $\tau$ from 1.25 to 0.75 invokes more corrections and further improves quality at some cost in speedup. For example, with $\mathcal{O}=4$, the FID decreases from 2.41 at $(\tau=1.25, r=0.3)$ to 2.17 at $(\tau=0.75, r=0.3)$, and reaches 2.16 at $(\tau=0.75, r=0.1)$ while still retaining $1.76\times$ acceleration. These results reinforce that higher-order prediction and more conservative correction thresholds trade a modest amount of speed for clearly better sample quality. Overall, these trends indicate that PrediT exposes a controllable Pareto frontier, allowing practitioners to choose a suitable operating point.

\section{Conclusion}

In this paper, we presented PrediT, a training-free acceleration framework for DiTs that formulates feature prediction as a linear multistep problem. Observing that model outputs evolve smoothly along the diffusion trajectory, we moved beyond naive feature reuse to principled prediction. Our framework combines an Adams-Bashforth predictor, an Adams-Moulton corrector, and dynamic step modulation to enable aggressive skipping in smooth regions while maintaining accuracy in high-dynamics regions. Experiments across diverse DiT-based models demonstrate that PrediT delivers substantial acceleration with strong output quality. One current limitation is that DSM operates at the full-model output level rather than at a finer block or layer granularity. An important direction for future work is to extend this control to finer-grained modules, which may further improve the accuracy-efficiency trade-off and broaden applicability to emerging architectures such as Mixture-of-Experts DiTs and interactive generation settings.


\bibliographystyle{ACM-Reference-Format}
\bibliography{sample-base}

@inproceedings{peebles2023scalable,
  title={Scalable diffusion models with transformers},
  author={Peebles, William and Xie, Saining},
  booktitle={Proceedings of the IEEE/CVF international conference on computer vision},
  pages={4195--4205},
  year={2023}
}

@inproceedings{ma2024deepcache,
  title={Deepcache: Accelerating diffusion models for free},
  author={Ma, Xinyin and Fang, Gongfan and Wang, Xinchao},
  booktitle={Proceedings of the IEEE/CVF conference on computer vision and pattern recognition},
  pages={15762--15772},
  year={2024}
}

@article{chen2024delta,
  title={$\Delta $-DiT: A Training-Free Acceleration Method Tailored for Diffusion Transformers},
  author={Chen, Pengtao and Shen, Mingzhu and Ye, Peng and Cao, Jianjian and Tu, Chongjun and Bouganis, Christos-Savvas and Zhao, Yiren and Chen, Tao},
  journal={arXiv preprint arXiv:2406.01125},
  year={2024}
}

@inproceedings{yu2025ab,
  title={Ab-cache: Training-free acceleration of diffusion models via adams-bashforth cached feature reuse},
  author={Yu, Zichao and Zou, Zhen and Shao, Guojiang and Zhang, Chenwei and Xu, Shengze and Huang, Jie and Zhao, Feng and Cun, Xiaodong and Zhang, Wenyi},
  booktitle={Proceedings of the 33rd ACM International Conference on Multimedia},
  pages={10408--10417},
  year={2025}
}

@article{liu2025reusing,
  title={From reusing to forecasting: Accelerating diffusion models with taylorseers},
  author={Liu, Jiacheng and Zou, Chang and Lyu, Yuanhuiyi and Chen, Junjie and Zhang, Linfeng},
  booktitle={Proceedings of the IEEE/CVF International Conference on Computer Vision},
  pages={15853--15863},
  year={2025}
}

@article{selvaraju2024fora,
  title={Fora: Fast-forward caching in diffusion transformer acceleration},
  author={Selvaraju, Pratheba and Ding, Tianyu and Chen, Tianyi and Zharkov, Ilya and Liang, Luming},
  journal={arXiv preprint arXiv:2407.01425},
  year={2024}
}

@article{flux2024,
  title={FLUX. 1 Kontext: Flow Matching for In-Context Image Generation and Editing in Latent Space},
  author={Labs, Black Forest and Batifol, Stephen and Blattmann, Andreas and Boesel, Frederic and Consul, Saksham and Diagne, Cyril and Dockhorn, Tim and English, Jack and English, Zion and Esser, Patrick and others},
  journal={arXiv preprint arXiv:2506.15742},
  year={2025}
}

@inproceedings{wangDiffusionDBLargescalePrompt2022,
  title={Diffusiondb: A large-scale prompt gallery dataset for text-to-image generative models},
  author={Wang, Zijie J and Montoya, Evan and Munechika, David and Yang, Haoyang and Hoover, Benjamin and Chau, Duen Horng},
  booktitle={Proceedings of the 61st annual meeting of the association for computational linguistics},
  pages={893--911},
  year={2023}
}

@article{kong2024hunyuanvideo,
  title={Hunyuanvideo: A systematic framework for large video generative models},
  author={Kong, Weijie and Tian, Qi and Zhang, Zijian and Min, Rox and Dai, Zuozhuo and Zhou, Jin and Xiong, Jiangfeng and Li, Xin and Wu, Bo and Zhang, Jianwei and others},
  journal={arXiv preprint arXiv:2412.03603},
  year={2024}
}

@article{song2020denoising,
  title={Denoising diffusion implicit models},
  author={Song, Jiaming and Meng, Chenlin and Ermon, Stefano},
  journal={arXiv preprint arXiv:2010.02502},
  year={2020}
}

@inproceedings{zhang2024evaluationagent,
  title={Evaluation agent: Efficient and promptable evaluation framework for visual generative models},
  author={Zhang, Fan and Tian, Shulin and Huang, Ziqi and Qiao, Yu and Liu, Ziwei},
  booktitle={Proceedings of the 63rd Annual Meeting of the Association for Computational Linguistics (Volume 1: Long Papers)},
  pages={7561--7582},
  year={2025}
}

@inproceedings{deng2009imagenet,
 title={Imagenet: A large-scale hierarchical image database},
  author={Deng, Jia and Dong, Wei and Socher, Richard and Li, Li-Jia and Li, Kai and Fei-Fei, Li},
  booktitle={2009 IEEE conference on computer vision and pattern recognition},
  pages={248--255},
  year={2009},
  organization={Ieee}
}

@inproceedings{radford2021learning,
  title={Learning transferable visual models from natural language supervision},
  author={Radford, Alec and Kim, Jong Wook and Hallacy, Chris and Ramesh, Aditya and Goh, Gabriel and Agarwal, Sandhini and Sastry, Girish and Askell, Amanda and Mishkin, Pamela and Clark, Jack and others},
  booktitle={International conference on machine learning},
  pages={8748--8763},
  year={2021},
  organization={PmLR}
}

@inproceedings{xu2023imagereward,
  title={Imagereward: Learning and evaluating human preferences for text-to-image generation},
  author={Xu, Jiazheng and Liu, Xiao and Wu, Yuchen and Tong, Yuxuan and Li, Qinkai and Ding, Ming and Tang, Jie and Dong, Yuxiao},
  journal={Advances in Neural Information Processing Systems},
  volume={36},
  pages={15903--15935},
  year={2023}
}

@inproceedings{dao2022flashattention,
  title={Flashattention: Fast and memory-efficient exact attention with io-awareness},
  author={Dao, Tri and Fu, Dan and Ermon, Stefano and Rudra, Atri and R{\'e}, Christopher},
  journal={Advances in neural information processing systems},
  volume={35},
  pages={16344--16359},
  year={2022}
}

@inproceedings{liu2022flow,
  title={Flow Straight and Fast: Learning to Generate and Transfer Data with Rectified Flow},
  author={Liu, Xingchao and Gong, Chengyue and others},
  booktitle={The Eleventh International Conference on Learning Representations},
  year={2022}
}

@inproceedings{zhang2018unreasonable,
  title={The unreasonable effectiveness of deep features as a perceptual metric},
  author={Zhang, Richard and Isola, Phillip and Efros, Alexei A and Shechtman, Eli and Wang, Oliver},
  booktitle={Proceedings of the IEEE conference on computer vision and pattern recognition},
  pages={586--595},
  year={2018}
}

@article{wang2004image,
  title={Image quality assessment: from error visibility to structural similarity},
  author={Wang, Zhou and Bovik, Alan C and Sheikh, Hamid R and Simoncelli, Eero P},
  journal={IEEE transactions on image processing},
  volume={13},
  number={4},
  pages={600--612},
  year={2004},
  publisher={IEEE}
}

@inproceedings{heusel2017gans,
  title={Gans trained by a two time-scale update rule converge to a local nash equilibrium},
  author={Heusel, Martin and Ramsauer, Hubert and Unterthiner, Thomas and Nessler, Bernhard and Hochreiter, Sepp},
  booktitle={Advances in neural information processing systems},
  volume={30},
  year={2017}
}

@inproceedings{salimans2016improved,
  title={Improved techniques for training gans},
  author={Salimans, Tim and Goodfellow, Ian and Zaremba, Wojciech and Cheung, Vicki and Radford, Alec and Chen, Xi},
  booktitle={Advances in neural information processing systems},
  volume={29},
  year={2016}
}

@inproceedings{liu2025smoothcache,
  title={Smoothcache: A universal inference acceleration technique for diffusion transformers},
  author={Liu, Joseph and Geddes, Joshua and Guo, Ziyu and Jiang, Haomiao and Nandwana, Mahesh Kumar},
  booktitle={Proceedings of the Computer Vision and Pattern Recognition Conference},
  pages={3229--3238},
  year={2025}
}

@InProceedings{ma2025model,
  title={Model reveals what to cache: Profiling-based feature reuse for video diffusion models},
  author={Ma, Xuran and Liu, Yexin and Liu, Yaofu and Wu, Xianfeng and Zheng, Mingzhe and Wang, Zihao and Lim, Ser-Nam and Yang, Harry},
  booktitle={Proceedings of the IEEE/CVF International Conference on Computer Vision},
  pages={17150--17159},
  year={2025}
}

@article{ma2025adams,
  title={Adams Bashforth Moulton Solver for Inversion and Editing in Rectified Flow},
  author={Ma, Yongjia and Di, Donglin and Liu, Xuan and Chen, Xiaokai and Fan, Lei and Su, Tonghua and Gao, Yue},
  journal={arXiv preprint arXiv:2503.16522},
  year={2025}
}

@inproceedings{liu2025timestep,
  title={Timestep Embedding Tells: It's Time to Cache for Video Diffusion Model},
  author={Liu, Feng and Zhang, Shiwei and Wang, Xiaofeng and Wei, Yujie and Qiu, Haonan and Zhao, Yuzhong and Zhang, Yingya and Ye, Qixiang and Wan, Fang},
  booktitle={Proceedings of the Computer Vision and Pattern Recognition Conference},
  pages={7353--7363},
  year={2025}
}

@article{fui2023generative,
  title={Generative AI and ChatGPT: Applications, challenges, and AI-human collaboration},
  author={Fui-Hoon Nah, Fiona and Zheng, Ruilin and Cai, Jingyuan and Siau, Keng and Chen, Langtao},
  journal={Journal of information technology case and application research},
  volume={25},
  number={3},
  pages={277--304},
  year={2023},
  publisher={Taylor \& Francis}
}

@article{cao2025survey,
  title={A survey of ai-generated content (aigc)},
  author={Cao, Yihan and Li, Siyu and Liu, Yixin and Yan, Zhiling and Dai, Yutong and Yu, Philip and Sun, Lichao},
  journal={ACM Computing Surveys},
  volume={57},
  number={5},
  pages={1--38},
  year={2025},
  publisher={ACM New York, NY}
}

@inproceedings{dhariwal2021diffusion,
  title={Diffusion models beat gans on image synthesis},
  author={Dhariwal, Prafulla and Nichol, Alexander},
  booktitle={Advances in neural information processing systems},
  volume={34},
  pages={8780--8794},
  year={2021}
}

@inproceedings{ho2020denoising,
  title={Denoising diffusion probabilistic models},
  author={Ho, Jonathan and Jain, Ajay and Abbeel, Pieter},
  booktitle={Advances in neural information processing systems},
  volume={33},
  pages={6840--6851},
  year={2020}
}

@inproceedings{rombach2022high,
  title={High-resolution image synthesis with latent diffusion models},
  author={Rombach, Robin and Blattmann, Andreas and Lorenz, Dominik and Esser, Patrick and Ommer, Bj{\"o}rn},
  booktitle={Proceedings of the IEEE/CVF conference on computer vision and pattern recognition},
  pages={10684--10695},
  year={2022}
}

@article{blattmann2023stable,
  title={Stable video diffusion: Scaling latent video diffusion models to large datasets},
  author={Blattmann, Andreas and Dockhorn, Tim and Kulal, Sumith and Mendelevitch, Daniel and Kilian, Maciej and Lorenz, Dominik and Levi, Yam and English, Zion and Voleti, Vikram and Letts, Adam and others},
  journal={arXiv preprint arXiv:2311.15127},
  year={2023}
}

@article{ramesh2022hierarchical,
  title={Hierarchical text-conditional image generation with clip latents},
  author={Ramesh, Aditya and Dhariwal, Prafulla and Nichol, Alex and Chu, Casey and Chen, Mark},
  journal={arXiv preprint arXiv:2204.06125},
  volume={1},
  number={2},
  pages={3},
  year={2022}
}

@article{opensora,
  title={Open-sora: Democratizing efficient video production for all},
  author={Zheng, Zangwei and Peng, Xiangyu and Yang, Tianji and Shen, Chenhui and Li, Shenggui and Liu, Hongxin and Zhou, Yukun and Li, Tianyi and You, Yang},
  journal={arXiv preprint arXiv:2412.20404},
  year={2024}
}

@article{chen2023pixart,
  title={Pixart-$\alpha$: Fast training of diffusion transformer for photorealistic text-to-image synthesis},
  author={Chen, Junsong and Yu, Jincheng and Ge, Chongjian and Yao, Lewei and Xie, Enze and Wu, Yue and Wang, Zhongdao and Kwok, James and Luo, Ping and Lu, Huchuan and others},
  journal={arXiv preprint arXiv:2310.00426},
  year={2023}
}

@inproceedings{chen2024pixart,
  title={Pixart-$\sigma$: Weak-to-strong training of diffusion transformer for 4k text-to-image generation},
  author={Chen, Junsong and Ge, Chongjian and Xie, Enze and Wu, Yue and Yao, Lewei and Ren, Xiaozhe and Wang, Zhongdao and Luo, Ping and Lu, Huchuan and Li, Zhenguo},
  booktitle={European Conference on Computer Vision},
  pages={74--91},
  year={2024},
  organization={Springer}
}

@article{liu2024sora,
  title={Sora: A review on background, technology, limitations, and opportunities of large vision models},
  author={Liu, Yixin and Zhang, Kai and Li, Yuan and Yan, Zhiling and Gao, Chujie and Chen, Ruoxi and Yuan, Zhengqing and Huang, Yue and Sun, Hanchi and Gao, Jianfeng and others},
  journal={arXiv preprint arXiv:2402.17177},
  year={2024}
}

@article{salimans2022progressive,
  title={Progressive distillation for fast sampling of diffusion models},
  author={Salimans, Tim and Ho, Jonathan},
  journal={arXiv preprint arXiv:2202.00512},
  year={2022}
}

@inproceedings{meng2023distillation,
  title={On distillation of guided diffusion models},
  author={Meng, Chenlin and Rombach, Robin and Gao, Ruiqi and Kingma, Diederik and Ermon, Stefano and Ho, Jonathan and Salimans, Tim},
  booktitle={Proceedings of the IEEE/CVF conference on computer vision and pattern recognition},
  pages={14297--14306},
  year={2023}
}

@inproceedings{shang2023post,
  title={Post-training quantization on diffusion models},
  author={Shang, Yuzhang and Yuan, Zhihang and Xie, Bin and Wu, Bingzhe and Yan, Yan},
  booktitle={Proceedings of the IEEE/CVF conference on computer vision and pattern recognition},
  pages={1972--1981},
  year={2023}
}

@inproceedings{li2023q,
  title={Q-diffusion: Quantizing diffusion models},
  author={Li, Xiuyu and Liu, Yijiang and Lian, Long and Yang, Huanrui and Dong, Zhen and Kang, Daniel and Zhang, Shanghang and Keutzer, Kurt},
  booktitle={Proceedings of the IEEE/CVF International Conference on Computer Vision},
  pages={17535--17545},
  year={2023}
}

@article{liu2025survey,
  title={A survey on cache methods in diffusion models: Toward efficient multi-modal generation},
  author={Liu, Jiacheng and Wang, Xinyu and Lin, Yuqi and Wang, Zhikai and Wang, Peiru and Cai, Peiliang and Zhou, Qinming and Yan, Zhengan and Yan, Zexuan and Shi, Zhengyi and others},
  journal={arXiv preprint arXiv:2510.19755},
  year={2025}
}

@article{chen2023videocrafter1,
  title={Videocrafter1: Open diffusion models for high-quality video generation},
  author={Chen, Haoxin and Xia, Menghan and He, Yingqing and Zhang, Yong and Cun, Xiaodong and Yang, Shaoshu and Xing, Jinbo and Liu, Yaofang and Chen, Qifeng and Wang, Xintao and others},
  journal={arXiv preprint arXiv:2310.19512},
  year={2023}
}

@inproceedings{chen2024videocrafter2,
  title={Videocrafter2: Overcoming data limitations for high-quality video diffusion models},
  author={Chen, Haoxin and Zhang, Yong and Cun, Xiaodong and Xia, Menghan and Wang, Xintao and Weng, Chao and Shan, Ying},
  booktitle={Proceedings of the IEEE/CVF conference on computer vision and pattern recognition},
  pages={7310--7320},
  year={2024}
}

@article{zhao2024real,
  title={Real-time video generation with pyramid attention broadcast},
  author={Zhao, Xuanlei and Jin, Xiaolong and Wang, Kai and You, Yang},
  journal={arXiv preprint arXiv:2408.12588},
  year={2024}
}

@article{croitoru2023diffusion,
  title={Diffusion models in vision: A survey},
  author={Croitoru, Florinel-Alin and Hondru, Vlad and Ionescu, Radu Tudor and Shah, Mubarak},
  journal={IEEE transactions on pattern analysis and machine intelligence},
  volume={45},
  number={9},
  pages={10850--10869},
  year={2023},
  publisher={Ieee}
}

@article{yang2023diffusion,
  title={Diffusion models: A comprehensive survey of methods and applications},
  author={Yang, Ling and Zhang, Zhilong and Song, Yang and Hong, Shenda and Xu, Runsheng and Zhao, Yue and Zhang, Wentao and Cui, Bin and Yang, Ming-Hsuan},
  journal={ACM computing surveys},
  volume={56},
  number={4},
  pages={1--39},
  year={2023},
  publisher={ACM New York, NY, USA}
}

@article{song2020score,
  title={Score-based generative modeling through stochastic differential equations},
  author={Song, Yang and Sohl-Dickstein, Jascha and Kingma, Diederik P and Kumar, Abhishek and Ermon, Stefano and Poole, Ben},
  journal={arXiv preprint arXiv:2011.13456},
  year={2020}
}

@book{butcher2016numerical,
  title={Numerical methods for ordinary differential equations},
  author={Butcher, John Charles},
  year={2016},
  publisher={John Wiley \& Sons}
}

@book{hairer1993solving,
  title={Solving ordinary differential equations I: Nonstiff problems},
  author={Hairer, Ernst and Wanner, Gerhard and N{\o}rsett, Syvert P},
  year={1993},
  publisher={Springer}
}

@article{lipman2022flow,
  title={Flow matching for generative modeling},
  author={Lipman, Yaron and Chen, Ricky TQ and Ben-Hamu, Heli and Nickel, Maximilian and Le, Matt},
  journal={arXiv preprint arXiv:2210.02747},
  year={2022}
}

@inproceedings{lu2022dpm,
  title={Dpm-solver: A fast ode solver for diffusion probabilistic model sampling in around 10 steps},
  author={Lu, Cheng and Zhou, Yuhao and Bao, Fan and Chen, Jianfei and Li, Chongxuan and Zhu, Jun},
  journal={Advances in neural information processing systems},
  volume={35},
  pages={5775--5787},
  year={2022}
}

@article{lu2025dpm,
  title={Dpm-solver++: Fast solver for guided sampling of diffusion probabilistic models},
  author={Lu, Cheng and Zhou, Yuhao and Bao, Fan and Chen, Jianfei and Li, Chongxuan and Zhu, Jun},
  journal={Machine Intelligence Research},
  volume={22},
  number={4},
  pages={730--751},
  year={2025},
  publisher={Springer}
}

@article{liu2022pseudo,
  title={Pseudo numerical methods for diffusion models on manifolds},
  author={Liu, Luping and Ren, Yi and Lin, Zhijie and Zhao, Zhou},
  journal={arXiv preprint arXiv:2202.09778},
  year={2022}
}

@inproceedings{karras2022elucidating,
  title={Elucidating the design space of diffusion-based generative models},
  author={Karras, Tero and Aittala, Miika and Aila, Timo and Laine, Samuli},
  booktitle={Advances in neural information processing systems},
  volume={35},
  pages={26565--26577},
  year={2022}
}

@inproceedings{karras2024analyzing,
  title={Analyzing and improving the training dynamics of diffusion models},
  author={Karras, Tero and Aittala, Miika and Lehtinen, Jaakko and Hellsten, Janne and Aila, Timo and Laine, Samuli},
  booktitle={Proceedings of the IEEE/CVF conference on computer vision and pattern recognition},
  pages={24174--24184},
  year={2024}
}

@inproceedings{ma2024learningtocache,
  title={Learning-to-cache: Accelerating diffusion transformer via layer caching},
  author={Ma, Xinyin and Fang, Gongfan and Bi Mi, Michael and Wang, Xinchao},
  booktitle={Advances in Neural Information Processing Systems},
  volume={37},
  pages={133282--133304},
  year={2024}
}

@inproceedings{esser2024scaling,
  title={Scaling rectified flow transformers for high-resolution image synthesis},
  author={Esser, Patrick and Kulal, Sumith and Blattmann, Andreas and Entezari, Rahim and M{\"u}ller, Jonas and Saini, Harry and Levi, Yam and Lorenz, Dominik and Sauer, Axel and Boesel, Frederic and others},
  booktitle={Forty-first international conference on machine learning},
  year={2024}
}
\clearpage
\appendix
\section{Adams Method Coefficients}
\label{app:coefficients}

\subsection{Derivation of Second-Order Adams-Bashforth (AB2)}

The Adams-Bashforth method approximates the integral:
\begin{equation*}
\int_{t_n}^{t_{n+1}} f(x(s), s) \, ds,
\end{equation*}
by constructing a Lagrange interpolating polynomial through historical function values. For AB2, we use the points $(t_n, f_n)$ and $(t_{n-1}, f_{n-1})$.

The Lagrange interpolating polynomial is:
\begin{equation}
P(s) = f_n \cdot \frac{s - t_{n-1}}{t_n - t_{n-1}} + f_{n-1} \cdot \frac{s - t_n}{t_{n-1} - t_n} = f_n \cdot \frac{s - t_{n-1}}{h} - f_{n-1} \cdot \frac{s - t_n}{h},
\end{equation}
where $h = \Delta t$ is the uniform step size. Integrating from $t_n$ to $t_{n+1}$:
\begin{align}
\int_{t_n}^{t_{n+1}} P(s) \, ds &= \frac{f_n}{h} \int_{t_n}^{t_{n+1}} (s - t_{n-1}) \, ds - \frac{f_{n-1}}{h} \int_{t_n}^{t_{n+1}} (s - t_n) \, ds \\
&= \frac{f_n}{h} \left[ \frac{(s - t_{n-1})^2}{2} \right]_{t_n}^{t_{n+1}} - \frac{f_{n-1}}{h} \left[ \frac{(s - t_n)^2}{2} \right]_{t_n}^{t_{n+1}} \\
&= \frac{f_n}{h} \cdot \frac{(2h)^2 - h^2}{2} - \frac{f_{n-1}}{h} \cdot \frac{h^2}{2} = \frac{3h}{2} f_n - \frac{h}{2} f_{n-1}.
\end{align}
Thus, the AB2 formula is:
\begin{equation}
x_{n+1} = x_n + \frac{h}{2}(3f_n - f_{n-1}),
\end{equation}
with coefficients $\beta_0 = \frac{3}{2}$ and $\beta_1 = -\frac{1}{2}$.

\subsection{Derivation of Second-Order Adams-Moulton (AM2)}

The Adams-Moulton method includes the future value $f_{n+1}$ in the interpolation, making it implicit. For the second-order case (AM2), we construct a polynomial through three points $(t_{n+1}, f_{n+1})$, $(t_n, f_n)$, and $(t_{n-1}, f_{n-1})$:
\begin{equation}
P(s) = f_{n+1} L_{-1}(s) + f_n L_0(s) + f_{n-1} L_1(s),
\end{equation}
where $L_j(s)$ are the Lagrange basis polynomials. After computing the integrals:
\begin{align}
  &\int_{t_n}^{t_{n+1}} L_{-1}(s) \, ds = \frac{5h}{12}, \\
  &\int_{t_n}^{t_{n+1}} L_0(s) \, ds = \frac{8h}{12}, \\
  &\int_{t_n}^{t_{n+1}} L_1(s) \, ds = -\frac{h}{12}.
\end{align}
Thus, the AM2 formula is:
\begin{equation}
x_{n+1} = x_n + \frac{h}{12}(5f_{n+1} + 8f_n - f_{n-1}),
\end{equation}
with coefficients $\gamma_{-1} = \frac{5}{12}$, $\gamma_0 = \frac{8}{12}$, and $\gamma_1 = -\frac{1}{12}$.

\subsection{Higher-Order Coefficients}

Following the same derivation procedure, we obtain coefficients for higher-order methods. \cref{tab:ab_coeff} and \cref{tab:am_coeff} summarize the coefficients and local truncation errors (LTE) for orders 1--4.

\begin{table}[h]
\centering
\caption{Adams-Bashforth coefficients and local truncation errors.}
\label{tab:ab_coeff}
\begin{tabular}{c|cccc|c}
\toprule
Order & $\beta_0$ & $\beta_1$ & $\beta_2$ & $\beta_3$ & LTE \\
\midrule
AB1 & $1$ & -- & -- & -- & $\frac{1}{2}h^2 f''$ \\
AB2 & $\frac{3}{2}$ & $-\frac{1}{2}$ & -- & -- & $\frac{5}{12}h^3 f'''$ \\
AB3 & $\frac{23}{12}$ & $-\frac{16}{12}$ & $\frac{5}{12}$ & -- & $\frac{9}{24}h^4 f^{(4)}$ \\
AB4 & $\frac{55}{24}$ & $-\frac{59}{24}$ & $\frac{37}{24}$ & $-\frac{9}{24}$ & $\frac{251}{720}h^5 f^{(5)}$ \\
\bottomrule
\end{tabular}
\end{table}

\begin{table}[h]
\centering
\caption{Adams-Moulton coefficients and local truncation errors.}
\label{tab:am_coeff}
\begin{tabular}{c|ccccc|c}
\toprule
Order & $\gamma_{-1}$ & $\gamma_0$ & $\gamma_1$ & $\gamma_2$ & $\gamma_3$ & LTE \\
\midrule
AM1 & $1$ & -- & -- & -- & -- & $-\frac{1}{2}h^2 f''$ \\
AM2 & $\frac{5}{12}$ & $\frac{8}{12}$ & $-\frac{1}{12}$ & -- & -- & $-\frac{1}{24}h^3 f'''$ \\
AM3 & $\frac{9}{24}$ & $\frac{19}{24}$ & $-\frac{5}{24}$ & $\frac{1}{24}$ & -- & $-\frac{19}{720}h^4 f^{(4)}$ \\
AM4 & $\frac{251}{720}$ & $\frac{646}{720}$ & $-\frac{264}{720}$ & $\frac{106}{720}$ & $-\frac{19}{720}$ & $-\frac{3}{160}h^5 f^{(5)}$ \\
\bottomrule
\end{tabular}
\end{table}

Note that the AM method of a given order has a smaller LTE coefficient than the AB method of the same order, providing improved accuracy at the cost of requiring $f_{n+1}$.

\section{Error Analysis}
\label{app:error}

\subsection{Discretization Error}

For an ODE $\dot{x} = f(x, t)$ with exact solution $x(t)$, the local truncation error (LTE) of a numerical method measures the error introduced in a single step assuming exact previous values.

\begin{proposition}[AB Local Truncation Error]
The Adams-Bashforth method of order $k$ has LTE:
\begin{equation}
\tau_n^{AB} = C_k h^{k+1} f^{(k)}(\xi) + \mathcal{O}(h^{k+2}),
\end{equation}
where $C_k$ is a constant depending on $k$, and $\xi \in [t_n, t_{n+1}]$.
\end{proposition}

\begin{proof}[Proof sketch]
The AB method of order $k$ approximates the integral:
\begin{equation*}
\int_{t_n}^{t_{n+1}} f(s) \, ds,
\end{equation*}
by interpolating $f$ at $t_n, t_{n-1}, \ldots, t_{n-k+1}$. Let $P_{k-1}$ denote the resulting degree-$(k-1)$ polynomial. The interpolation error is:
\begin{equation}
f(s) - P_{k-1}(s) = \frac{f^{(k)}(\xi_s)}{k!} \prod_{j=0}^{k-1}(s - t_{n-j}),
\end{equation}
for some $\xi_s$ between the extremes of $\{s, t_n, \ldots, t_{n-k+1}\}$. Integrating the remainder over $[t_n, t_{n+1}]$:
\begin{equation}
\tau_n = \int_{t_n}^{t_{n+1}} \frac{f^{(k)}(\xi_s)}{k!} \prod_{j=0}^{k-1}(s - t_{n-j}) \, ds.
\end{equation}
For uniform step size $h$, substitute $s = t_n + \sigma h$ with $\sigma \in [0, 1]$. Since each factor satisfies $(s - t_{n-j}) = (\sigma + j)h$, we obtain:
\begin{align}
\tau_n &= \frac{h^{k+1} f^{(k)}(\xi)}{k!} \int_0^1 \prod_{j=0}^{k-1}(\sigma + j) \, d\sigma \\
&= C_k h^{k+1} f^{(k)}(\xi),
\end{align}
where:
\begin{equation*}
C_k = \frac{1}{k!}\int_0^1 \prod_{j=0}^{k-1}(\sigma + j) \, d\sigma.
\end{equation*}
For example, AB1/Euler gives $C_1 = 1/2$. AB2 gives $C_2 = 5/12$.
\end{proof}

The global error after $N$ steps is bounded by:
\begin{equation}
\|x_N - x(t_N)\| \leq \frac{e^{LT} - 1}{L} \max_n \|\tau_n\| = \mathcal{O}(h^k),
\end{equation}
where $L$ is the Lipschitz constant of $f$ and $T$ is the total integration time.

\subsection{Prediction Error During Skips}

When skipping model calls, we use AB to predict $f_{n+1}$ from history:
\begin{equation}
\hat{f}_{n+1} = \sum_{j=0}^{k-1} \alpha_j f_{n-j},
\end{equation}
where $\alpha_j$ are extrapolation coefficients. The prediction error is:
\begin{equation}
\epsilon_{\text{pred}} = \|f_{n+1} - \hat{f}_{n+1}\| \approx \mathcal{O}(h^k \|f^{(k)}\|).
\end{equation}

This error depends on the local curvature of the trajectory. Our dynamics metric $\delta_n$ approximates the first derivative, providing a proxy for curvature.

\subsection{Error Accumulation and Mitigation}

When $J$ consecutive steps are skipped, errors accumulate:
\begin{equation}
\epsilon_{\text{acc}}(J) \approx \sum_{i=1}^{J} \epsilon_{\text{pred}}(i) \cdot (1 + Lh)^{J-i}.
\end{equation}

To see this, note that at each skipped step $i$, the prediction error $\epsilon_{\text{pred}}(i)$ is introduced into the latent. Since the ODE has Lipschitz constant $L$, any perturbation at step $i$ is amplified by at most a factor $(1+Lh)$ per subsequent step. After $J-i$ additional steps, the error from step $i$ has grown by $(1+Lh)^{J-i}$. Summing over all skipped steps yields the geometric series above. For small $Lh$, this simplifies to $\epsilon_{\text{acc}}(J) \approx J \cdot \epsilon_{\text{pred}} \cdot e^{LJh}$, showing that the accumulated error grows roughly exponentially with the skip interval $J$, motivating the adaptive control provided by DSM.

The exponential factor $(1 + Lh)^{J-i}$ reflects error propagation through the ODE. Our framework mitigates this through:

\begin{enumerate}
   \item \textbf{Adaptive skip interval}: When $\delta_n$ is large, $J$ is automatically reduced via \cref{eq:jump}.
   \item \textbf{ABM correction}: When $\delta_n \geq \tau$, we invoke ABM, which evaluates the true $f_{n+1}$ and resets the prediction error to the discretization error of AM.
   \item \textbf{Threshold-based switching}: The correction ratio $r$ creates a buffer zone where ABM is used even for moderate dynamics.
\end{enumerate}

\begin{proposition}[Error Bound with Dynamic Step Modulation]
Under the assumption that $\|f^{(k)}\| \leq M$ and with threshold $\tau$ chosen such that $\delta_n < \tau$ implies $\|f'\| \leq \tau \|f\|$, the accumulated error over a trajectory satisfies:
\begin{equation}
\|\epsilon_{\text{total}}\| \leq C_1 h^k + C_2 \tau J_{\max} + C_3 N_{\text{ABM}},
\end{equation}
where $J_{\max}$ is the maximum allowed skip interval and $N_{\text{ABM}}$ is the number of ABM corrections that reset the accumulated drift.
\end{proposition}

This shows that the error is controlled by the interplay between the method order $k$, the threshold $\tau$, and the frequency of corrections.

\section{Algorithm Pseudocode}
\label{app:algorithm}

\begin{algorithm}[h]
\caption{PrediT: Predictive Diffusion Transformer Inference}
\label{alg:predit}
\begin{algorithmic}
\REQUIRE Noise $x_1$, timesteps $\{t_n\}_{n=0}^{N}$, model $f_\theta$, order $k$, threshold $\tau$, ratio $r$
\STATE Initialize history buffer $\mathcal{H} \leftarrow \emptyset$, skip counter $J \leftarrow 0$
\FOR{$n = 0$ to $N-1$}
   \STATE $\Delta t \leftarrow t_{n+1} - t_n$
   \IF{$J > 0$}
       \STATE $x \leftarrow x + \text{AB}_k(\Delta t, \mathcal{H})$  \texttt{// Skip: use cached history}
       \STATE $J \leftarrow J - 1$
   \ELSE
       \STATE $f_n \leftarrow f_\theta(x, t_n)$ \texttt{// Compute model output}
       \STATE $\delta_n \leftarrow \|f_n - f_{n-1}\|_1 / (\|f_n\|_1 + \epsilon)$ \texttt{// Dynamics metric}
       \IF{$\delta_n \geq \tau$}
           \STATE $x \leftarrow \text{ABM}_k(x, \Delta t, f_n, \mathcal{H}, f_\theta)$ \texttt{// Corrected by ABM}
       \ELSIF{$\delta_n \geq \tau \cdot r$}
           \STATE $x \leftarrow \text{ABM}_k(x, \Delta t, f_n, \mathcal{H}, f_\theta)$, $J \leftarrow \text{ComputeSkip}(\delta_n, \tau)$ \texttt{// Corrected by ABM with dynamic step modulation}
       \ELSE
           \STATE $x \leftarrow x + \text{AB}_k(\Delta t, [f_n] \cup \mathcal{H})$, $J \leftarrow \text{ComputeSkip}(\delta_n, \tau)$ \texttt{// Predicted by AB with dynamic step modulation}
       \ENDIF
       \STATE Update $\mathcal{H} \leftarrow [f_n] \cup \mathcal{H}$
   \ENDIF
\ENDFOR
\STATE \textbf{return} $x$
\end{algorithmic}
\end{algorithm}

\begin{table*}[t]
  \centering
  \small
  \caption{Memory Usage (MB) on Different Resolutions and Frames using A800 80GB.}
  \label{tab:memory_usage}
  \begin{tabular}{c|ccccc}
  \toprule
  \textbf{Method} & \textbf{480p$\times$640p$\times$45f} & \textbf{480p$\times$640p$\times$65f} & \textbf{544p$\times$960p$\times$17f} & \textbf{544p$\times$960p$\times$65f} & \textbf{720p$\times$1280p$\times$129f} \\ \midrule
  Original     & 44295 & 45949 & 43131 & 49477 & 74465 \\
  PrediT       & 44819 & 46313 & 44043 & 49715 & 74525 \\
  PAB          & 79831 & OOM   & 55045 & OOM   & OOM \\
  TeaCache     & 48619 & 51499 & 48227 & 53319 & 80083 \\
  ProfilingDiT & OOM   & OOM   & 57237 & OOM   & OOM \\
  TaylorSeer   & OOM   & OOM   & 61557 & OOM   & OOM \\ \bottomrule
  \end{tabular}
 \end{table*}

 \begin{table*}[t]
  \centering
  \small
  \caption{Ablation Study with Different Configurations on ImageNet with DiT-XL/2.}
  \label{tab:ablation_study}
  \begin{tabular}{c|cc|ccc}
  \toprule
  \textbf{Settings} & \textbf{Latency(s) $\downarrow$} & \textbf{Speedup $\uparrow$} & \textbf{FID$\downarrow$} & \textbf{sFID$\downarrow$} & \textbf{IS$\uparrow$} \\ \midrule
  \textbf{DDIM-50} & 0.569 & 1.00$\times$ & 2.28 & 4.25 & 242.3 \\
  \textbf{($\mathcal{O}=2$, $\tau=1.25$, $r=0.3$)} & 0.219 & 2.60$\times$ & 2.97 & 5.33 & 222.6 \\
  \textbf{($\mathcal{O}=2$, $\tau=1.25$, $r=0.1$)} & 0.250 & 2.28$\times$ & 3.01 & 5.24 & 222.5 \\
  \textbf{($\mathcal{O}=2$, $\tau=1.00$, $r=0.3$)} & 0.251 & 2.27$\times$ & 2.71 & 5.31 & 228.4 \\
  \textbf{($\mathcal{O}=2$, $\tau=1.00$, $r=0.2$)} & 0.297 & 1.92$\times$ & 2.72 & 5.28 & 228.3 \\
  \textbf{($\mathcal{O}=2$, $\tau=1.00$, $r=0.1$)} & 0.293 & 1.94$\times$ & 2.78 & 5.30 & 227.7 \\
  \textbf{($\mathcal{O}=2$, $\tau=0.75$, $r=0.3$)} & 0.291 & 1.96$\times$ & 2.44 & 4.84 & 235.0 \\
  \textbf{($\mathcal{O}=2$, $\tau=0.75$, $r=0.1$)} & 0.306 & 1.86$\times$ & 2.48 & 4.88 & 234.3 \\ \midrule
  \textbf{($\mathcal{O}=3$, $\tau=1.25$, $r=0.3$)} & 0.228 & 2.50$\times$ & 2.54 & 5.12 & 231.1 \\
  \textbf{($\mathcal{O}=3$, $\tau=1.25$, $r=0.1$)} & 0.259 & 2.20$\times$ & 2.49 & 4.85 & 232.4 \\
  \textbf{($\mathcal{O}=3$, $\tau=1.00$, $r=0.3$)} & 0.254 & 2.24$\times$ & 2.39 & 4.89 & 235.8 \\
  \textbf{($\mathcal{O}=3$, $\tau=1.00$, $r=0.2$)} & 0.286 & 1.99$\times$ & 2.39 & 4.89 & 235.8 \\
  \textbf{($\mathcal{O}=3$, $\tau=1.00$, $r=0.1$)} & 0.287 & 1.98$\times$ & 2.37 & 4.77 & 237.8 \\
  \textbf{($\mathcal{O}=3$, $\tau=0.75$, $r=0.3$)} & 0.295 & 1.93$\times$ & 2.25 & 4.70 & 239.7 \\
  \textbf{($\mathcal{O}=3$, $\tau=0.75$, $r=0.1$)} & 0.312 & 1.82$\times$ & 2.25 & 4.65 & 239.9 \\ \midrule
  \textbf{($\mathcal{O}=4$, $\tau=1.25$, $r=0.3$)} & 0.240 & 2.37$\times$ & 2.41 & 5.12 & 234.7 \\
  \textbf{($\mathcal{O}=4$, $\tau=1.25$, $r=0.1$)} & 0.269 & 2.12$\times$ & 2.33 & 4.82 & 236.8 \\
  \textbf{($\mathcal{O}=4$, $\tau=1.00$, $r=0.3$)} & 0.268 & 2.12$\times$ & 2.24 & 4.78 & 239.5 \\
  \textbf{($\mathcal{O}=4$, $\tau=1.00$, $r=0.2$)} & 0.302 & 1.88$\times$ & 2.24 & 4.75 & 239.8 \\
  \textbf{($\mathcal{O}=4$, $\tau=1.00$, $r=0.1$)} & 0.309 & 1.84$\times$ & 2.23 & 4.60 & 240.1 \\
  \textbf{($\mathcal{O}=4$, $\tau=0.75$, $r=0.3$)} & 0.305 & 1.87$\times$ & 2.17 & 4.61 & 242.7 \\
  \textbf{($\mathcal{O}=4$, $\tau=0.75$, $r=0.1$)} & 0.324 & 1.76$\times$ & 2.16 & 4.49 & 243.3 \\ \bottomrule
  \end{tabular}
  \end{table*}
\section{Evaluation Metrics}
\label{app:metrics}

\subsection{Text-to-Image Metrics}

\textbf{ImageReward.} ImageReward measures the alignment between generated images and human preferences. It is trained on human preference data to predict which image better matches the given text prompt:
\begin{equation}
\text{ImageReward}(x, c) = f_\phi(x, c),
\end{equation}
where $x$ is the generated image, $c$ is the text prompt, and $f_\phi$ is the reward model trained on human annotations.

\textbf{CLIP Score.} CLIP Score evaluates the semantic consistency between the generated image and the input text prompt using a pretrained CLIP model:
\begin{equation}
\text{CLIPScore}(x, c) = \max(100 \cdot \cos(E_I(x), E_T(c)), 0),
\end{equation}
where $E_I$ and $E_T$ are the image and text encoders of CLIP, and $\cos(\cdot, \cdot)$ denotes cosine similarity.

\textbf{Aesthetic Score.} Aesthetic Score predicts the visual appeal of generated images using a model trained on human aesthetic ratings:
\begin{equation}
\text{AestheticScore}(x) = g_\psi(E_I(x)),
\end{equation}
where $g_\psi$ is a linear regressor trained on the Aesthetic Visual Analysis (AVA) dataset aesthetic annotations.

\subsection{Text-to-Video Metrics}

\textbf{VBench.} VBench is a comprehensive benchmark that evaluates video generation across multiple dimensions, including subject consistency, background consistency, motion smoothness, aesthetic quality, and imaging quality. The overall score is a weighted combination of these sub-metrics.

\textbf{LPIPS.} Learned Perceptual Image Patch Similarity measures perceptual similarity between frames:
\begin{equation}
\text{LPIPS}(x, y) = \sum_l \frac{1}{H_l W_l} \sum_{h,w} \|w_l \odot (\hat{\phi}_l^{hw}(x) - \hat{\phi}_l^{hw}(y))\|_2^2,
\end{equation}
where $\hat{\phi}_l$ are normalized activations from layer $l$ of a pretrained network and $w_l$ are learned weights.

\textbf{SSIM.} Structural Similarity Index measures the structural similarity between two images:
\begin{equation}
\text{SSIM}(x, y) = \frac{(2\mu_x\mu_y + c_1)(2\sigma_{xy} + c_2)}{(\mu_x^2 + \mu_y^2 + c_1)(\sigma_x^2 + \sigma_y^2 + c_2)},
\end{equation}
where $\mu_x, \mu_y$ are local means, $\sigma_x, \sigma_y$ are standard deviations, $\sigma_{xy}$ is the covariance, and $c_1, c_2$ are stability constants.

\textbf{PSNR.} Peak Signal-to-Noise Ratio measures the reconstruction quality:
\begin{equation}
\text{PSNR}(x, y) = 10 \cdot \log_{10}\left(\frac{\text{R}^2}{\text{MSE}(x, y)}\right),
\end{equation}
where $R$ is the maximum possible pixel value, and MSE is the mean squared error.

\subsection{Class-to-Image Metrics}

\textbf{FID.} Fr\'echet Inception Distance measures the similarity between the distributions of generated and real images:
\begin{equation}
\text{FID} = \|\mu_r - \mu_g\|^2 + \text{Tr}(\Sigma_r + \Sigma_g - 2(\Sigma_r \Sigma_g)^{1/2}),
\end{equation}
where $(\mu_r, \Sigma_r)$ and $(\mu_g, \Sigma_g)$ are the mean and covariance of Inception features for real and generated images.

\textbf{sFID.} Spatial FID computes FID on spatial feature maps rather than pooled features, better capturing spatial structure.

\textbf{Inception Score.} Inception Score evaluates both quality and diversity of generated samples:
\begin{equation}
\text{IS} = \exp\left(\mathbb{E}_x [D_{KL}(p(y|x) \| p(y))]\right),
\end{equation}
where $p(y|x)$ is the conditional class distribution from an Inception classifier and $p(y) = \mathbb{E}_x[p(y|x)]$ is the marginal distribution. Higher IS indicates both confident class predictions and diverse class coverage.

\section{Text-to-Image Prompts}
\label{app:prompts}

The prompts used for the visual comparisons in \cref{fig:flux} are listed below.

\begin{enumerate}
   \item \textbf{Column 1:} [``rat with bandages on bandaged arms holding wakizashi, highly detailed, d \& d, fantasy, highly detailed, digital painting, trending on artstation, concept art, sharp focus, illustration, global illumination, shaded, art by artgerm and greg rutkowski and fuji choko and viktoria gavrilenko and hoang lap"]
   \item \textbf{Column 2:} [``anna from frozen | | fine detail!! anime!! realistic shaded lighting!!, directional lighting 2 d poster by ilya kuvshinov, magali villeneuve, artgerm, jeremy lipkin and michael garmash and rob rey"]
   \item \textbf{Column 3:} [``a boy holding on to a dying old dog connecting him to his childhood"]
   \item \textbf{Column 4:} [``a beautiful ultradetailed anime illustration of unknown backroom level nature by bjarke ingels, landscape sunlight laser studio ghibli infrared cyberpunk at night synthwave tokyo vaporwave neon noir magic realism, archdaily, wallpaper, highly detailed, trending on artstation."]
   \item \textbf{Column 5:} [``the moon as a lightbulb being unscrewed by someones hand"]
\end{enumerate}

\section{More Visual Results}
\label{app:more_visual}

\subsection{More Visual Comparisons on FLUX.1}

\cref{fig:flux_more} presents additional visual comparisons on FLUX.1. While reuse-based methods exhibit visible degradation in fine details and prediction-based methods show subtle artifacts, PrediT preserves sharp details and accurate color reproduction across diverse prompts, validating the effectiveness of our approach.

\begin{figure*}[t]
 \centering
 \includegraphics[width=0.96\linewidth,height=0.74\textheight,keepaspectratio]{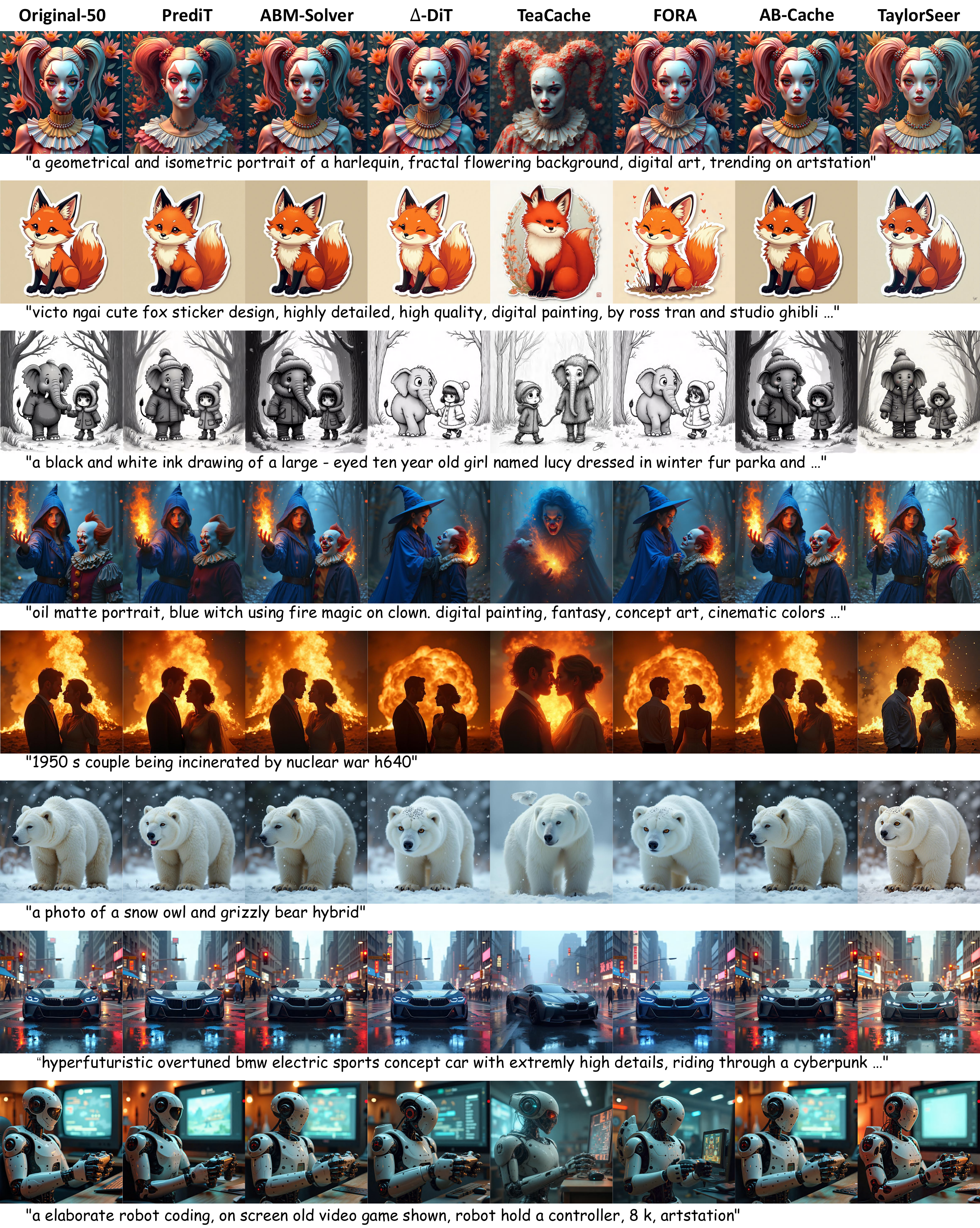}
 \caption{More Visual Comparisons on FLUX.1. Zoom in for best viewing.}
 \Description{A supplementary comparison grid of FLUX text-to-image results across additional prompts, showing the original model, baseline accelerators, and PrediT.}
 \label{fig:flux_more}
\end{figure*}

\subsection{More Visual Comparisons on HunyuanVideo}

\cref{fig:hy_more} provides additional visual comparisons on HunyuanVideo. Video generation is more challenging as errors can accumulate across both denoising steps and video frames. Existing methods show blurriness or temporal inconsistencies, while PrediT maintains sharp spatial details and smooth temporal transitions, achieving up to $3.28\times$ speedup while preserving visual quality.

\begin{figure*}[t]
 \centering
 \includegraphics[width=0.96\linewidth,height=0.74\textheight,keepaspectratio]{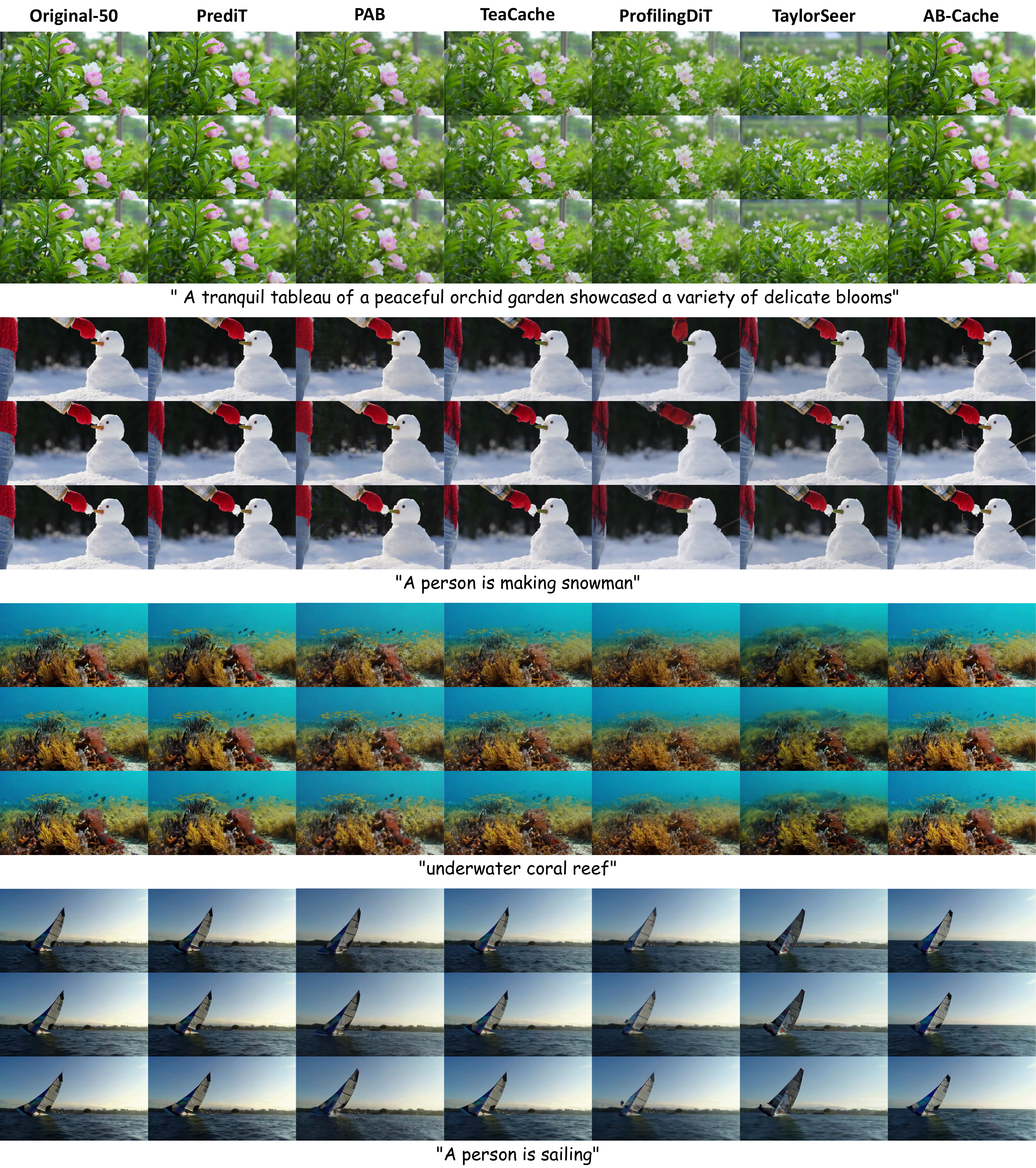}
 \caption{More Visual Comparisons on HunyuanVideo. Zoom in for best viewing.}
 \Description{A supplementary comparison grid of HunyuanVideo frames across additional prompts, contrasting baseline acceleration methods with PrediT.}
 \label{fig:hy_more}
\end{figure*}

\section{Memory Usage Analysis}
\label{app:memory}

\textbf{Theoretical Space Complexity.}
PrediT stores only the $k$ most recent final model outputs $\{f_n, f_{n-1}, \ldots, f_{n-k+1}\}$ for the AB/AM computation, where $k$ is the predictor order (typically $k \leq 4$). Each output has the same shape as a single model forward pass result. Therefore, the additional memory overhead is $\mathcal{O}(k \cdot d)$, where $d$ is the dimensionality of the model output. In contrast, methods like TaylorSeer~\cite{liu2025reusing} store Taylor expansion coefficients \emph{per layer}, requiring $\mathcal{O}(L \cdot d_l)$ additional memory where $L$ is the number of layers and $d_l$ includes intermediate feature dimensions (typically much larger than $d$). PAB~\cite{zhao2024real} caches full attention maps, requiring $\mathcal{O}(N^2)$ memory per cached layer. This fundamental difference in what is cached explains why PrediT introduces only 1--2\% memory overhead while other methods cause out-of-memory errors at high resolutions.

Video generation with DiT models requires substantial GPU memory due to the large spatial-temporal feature maps. We therefore analyze the memory consumption of different acceleration methods on HunyuanVideo across various resolutions and frame settings using a single NVIDIA A800 80GB GPU. As shown in \cref{tab:memory_usage}, PrediT introduces minimal memory overhead compared to the original model, with only 1\%-2\% additional memory usage across all settings. In contrast, prediction-based methods require storing Taylor expansion coefficients and historical gradients, leading to out-of-memory (OOM) errors at most settings. PAB also suffers from OOM at higher resolutions due to its attention caching mechanism. TeaCache shows moderate memory increase but remains functional. These results demonstrate that PrediT achieves significant speedup while maintaining memory efficiency, making it practical for real-world video generation applications.

\section{Supplementary Results for Ablation Studies}
\label{app:ablation_studies}

We provide additional ablation studies on DiT-XL/2 to analyze the impact of different hyperparameter configurations. \cref{tab:ablation_study} shows results with varying predictor order, dynamics threshold, and correction ratio. Higher-order predictors achieve better quality at the cost of reduced speedup, while smaller $\tau$ values lead to more frequent corrections and improved quality. Notably, with $\mathcal{O}=4$ and $\tau=0.75$, PrediT achieves an FID of 2.17, surpassing the baseline and demonstrating that our method can improve quality while accelerating inference.



\end{document}